\documentclass{article}

\usepackage{amssymb}  % \checkmark
\usepackage{amsmath}   % \text

\usepackage{amsthm} % jose: proof
\newtheorem{theorem}{Theorem}
\newtheorem{proposition}[theorem]{Proposition}

\newtheorem{example}{Example}
\newtheorem{definition}{Definition}

\usepackage{color}
\definecolor{orange}{RGB}{255,127,0}
\definecolor{brown}{RGB}{150,70,0}
\definecolor{green}{RGB}{127,255,127}
\definecolor{darkgreen}{RGB}{0,127,0}
\definecolor{blue}{RGB}{127,127,255}
\definecolor{lightblue}{RGB}{150,150,255}
\definecolor{darkblue}{RGB}{0,0,127}
\definecolor{red}{RGB}{255,90,90}
\definecolor{grey}{RGB}{127,127,127}
\definecolor{pink}{RGB}{255,180,180}

\newcommand{\comment}[1]{}

\usepackage{graphicx}

\usepackage{xspace}	
\newcommand{\xaxis}{$x$-axis\xspace}
\newcommand{\yaxis}{$y$-axis\xspace}

\newcommand{\JROC}{JROC\xspace}

\title{Test cost and misclassification cost trade-off using reframing}
\author
{
	Celestine Periale Maguedong-Djoumessi\\
	{\normalsize\em DSIC, Universitat Polit\`ecnica de Val\`encia, Spain}\\
	{\normalsize \tt cemadj@gmail.com}\\
$\:$\\
	Jos\'{e} Hern\'{a}ndez-Orallo\\
	{\normalsize\em DSIC, Universitat Polit\`ecnica de Val\`encia, Spain}\\
	{\normalsize \tt jorallo@dsic.upv.es}
}
\date{}

\begin{document}

\maketitle

{
\abstract Many solutions to cost-sensitive classification (and regression) rely on some or all of the following assumptions: we have complete knowledge about the cost context at training time, we can easily re-train whenever the cost context changes, and we have technique-specific methods (such as cost-sensitive decision trees) that can take advantage of that information. In this paper we address the problem of selecting models and minimising joint cost (integrating both misclassification cost and test costs) without any of the above assumptions. We introduce methods and plots (such as the so-called \JROC plots) that can work with any off-the-shelf predictive technique, including ensembles, such that we reframe the model to use the appropriate subset of attributes (the {\em feature configuration}) during deployment time. In other words, models are trained with the available attributes (once and for all) and then deployed by setting missing values on the attributes that are deemed ineffective for reducing the joint cost. As the number of feature configuration combinations grows exponentially with the number of features we introduce quadratic methods that are able to approximate the optimal configuration and model choices, as shown by the experimental results.  
}

%%%%%%%%%%%%%%%%%%%%%%%%%%%%%%%%%%%%%%%%%%%%%%%%%%%%%%%%%%%%%%%%%%
%%%%%%%%%%%%%%%%%%%%%%%%%%%%%%%%%%%%%%%%%%%%%%%%%%%%%%%%%%%%%%%%%%
\section{Introduction}
%%%%%%%%%%%%%%%%%%%%%%%%%%%%%%%%%%%%%%%%%%%%%%%%%%%%%%%%%%%%%%%%%%
%%%%%%%%%%%%%%%%%%%%%%%%%%%%%%%%%%%%%%%%%%%%%%%%%%%%%%%%%%%%%%%%%%

The feature space (including both input and output variables) characterises a data mining problem \cite{li2001feature}. In predictive (supervised) problems, the quality and availability of features determines the predictability of the dependent variable, and the performance of data mining models in terms of misclassification or regression error. Good features, however, are usually difficult to obtain. It is usual that many instances come with missing values, either because the actual value for a given attribute was not available or because it was too expensive (e.g., in medical domains, where attributes usually correspond to diagnostic tests). This frequently represents a utility or cost-sensitive learning dilemma \cite{turney2000types,elkan2001foundations} between misclassification (or regression error) costs and tests costs, both being integrated into a joint cost.

One possible option is known as missing value imputation \cite{zhu2011missing}, but this approach is not usually appropriate when test costs are considered. First, imputing missing values 
%''The approaches we discuss in this paper do not impute any missing values as it 
``is regarded as unnecessary for cost-sensitive learning that also considers the test costs'' \cite{zhang2005missing}. Second, expensive attributes (e.g., in diagnosis) are usually missing for many other instances as well and it is difficulty to infer them from other instances or attributes. 

The most common option is to train models that are able to do reasonably good predictions with the available attributes. However, a more powerful approach is to find a trade-off (in terms of minimising joint cost) about how many (and which) attributes need to be used. Retraining with all the attribute subsets (possibly using feature selection)
% \cite{guyon2003introduction,molina2002feature}. 
 does not seem to be a good option, because for $n$ attributes we typically have $2^n$ possible combinations. 
%Also, only using the training instances for which the same subset of attributes is available can bias each `partial' model. 
As a result, one common option is to use techniques that lead to models that can use any subset of attributes. Decision trees are the usual choice \cite{ling2004decision,zhang2005missing,lomax2013survey} because the use of attributes can be customised in many different ways. Similarly, we could also try ---if not already done--- to design cost-sensitive versions of many other families of techniques, such as Bayesian models, neural networks, logistic regression, kernel methods, etc., with varying success. This would lead to two problems. One one hand, we would need to have a library of specific cost-sensitive algorithms for classification and regression, which would also limit our range of options and the use of the ultimate learning techniques (until cost-sensitive versions appear and are implemented). On the other hand, even assuming that this is possible, we would require some tools to properly select which model is better, as we do not know in advance (in training time) what the misclassification (or regression error) cost and test cost context will be during deployment. In fact, each instance may have a different subset of missing values and a different cost context, so any choice performed during the training stage will be specific and biased.

In this paper, we explore an alternative, more general approach that can use off-the-shelf machine learning methods. The procedure is simple: we use any data mining technique that accepts missing values during training and prediction and learn a predictive model with our training data as usual. Then, 
we evaluate the model (on a validation dataset) by exploring the lattice of attribute subsets, with a very straightforward mechanism: we set missing values on purpose for each combination in the lattice. As a result, we know how well our model behaves for any attribute subset. 
From here, once the model needs to be deployed on unlabelled data, and whenever a new instance appears (with a possibly particular cost context) we decide which attributes the model requires to get the lowest expected joint cost\footnote{Given an example with some non-missing and some missing values we may decide increase the number of non-missing values. Given a case for which we have not still retrieved any of the attributes (tests) we decide how many (and which) attributes we are going to ask for.}. This is done by calculating the expected joint cost for each point in the lattice. In this sense, each prediction is associated with a possibly different operating condition, and the best attribute subset is chosen.

Interestingly, we can use the previous approach for more than one model, and see that some models dominate for some operating conditions over the rest. This is exactly the way ROC analysis works (for classification \cite{SDM00,flach2003decision,rocai2004,Fawcett06,Mamitsuka2006} and for regression \cite{RROC2012}). We will introduce graphical plots and procedures to make this selection and also to reduce the number of combinations in the lattice that need to be explored in order to make a good selection. 

The goal of the paper is then to introduce new methods to make optimal choices in terms of joint cost (i.e., considering both misclassification and test costs) when using off-the-shelf data mining models. An optimal choice is understood as selecting the right model with the right subset of attributes. 
In this paper we will focus on classification, but many of the ideas could be extended to regression as well.

Section \ref{motivation} reviews the notion of misclassification and test costs, cost context, and motivates the possibility of solving the problem of cost minimisation and model selection in a different way. 
Section \ref{reframing} proposes the idea of reframing an existing model to a new cost context by setting (or letting) some of the attributes as missing. Several examples show that good results can be obtained by setting most of the attributes missing during deployment, and show that the approach can be applied to any kind of predictive technique, including ensembles.
Section \ref{tradeoff} introduces a more effective way of analysing and finding the trade-off between test cost ($TC$) and misclassification cost ($MC$), by plotting $TC$ on the \xaxis and $MC$ on the \yaxis, and finding the optimal feature and model configuration using isometrics on these plots. The notion of convex hull and dominance are also introduced.
After realising that there are $2^m$ feature configurations for $m$ features, section \ref{hull} explores hull approximations by performing a quadratic selection on the number of configurations that need to be explored.
Section \ref{experiments} evaluates these approximations on several datasets and cost context.
Section \ref{conclusion} closes the paper with some recommendations and take-away messages about how to use the methods and plots introduced here. Several extensions are proposed as future work.

%%%%%%%%%%%%%%%%%%%%%%%%%%%%%%%%%%%%%%%%%%%%%%%%%%%%%%%%%%%%%%%%%%
%%%%%%%%%%%%%%%%%%%%%%%%%%%%%%%%%%%%%%%%%%%%%%%%%%%%%%%%%%%%%%%%%%
\section{Motivation}\label{motivation}
%%%%%%%%%%%%%%%%%%%%%%%%%%%%%%%%%%%%%%%%%%%%%%%%%%%%%%%%%%%%%%%%%%
%%%%%%%%%%%%%%%%%%%%%%%%%%%%%%%%%%%%%%%%%%%%%%%%%%%%%%%%%%%%%%%%%%

We will focus on classification problems, characterised by a multivariate input domain $\mathbb{X}$, i.e., a tuple of elements of sets $X_1, X_2, \dots, X_m$, where $m$ is the number of features or input attributes, possibly containing the null value, and a univariate output domain $\mathbb{Y} \subset \{l_1, l_2, \dots, l_c\}$, where $c$ is the number of classes or labels of the output attribute. The domain space $\mathbb{D}$ is then $\mathbb{X} \times \mathbb{Y}$. 
Examples or instances are just pairs $\left\langle x,y \right\rangle \in \mathbb{D}$, and datasets are subsets (actually multi-sets) of $\mathbb{D}$. The length of a dataset will usually be denoted by $n$. 
A {\em crisp} or {\em hard} classification model $\hat{f}$ is a function $\hat{f}: \mathbb{X} \rightarrow \mathbb{Y}$. 
We just represent the true value by $y$ and the estimated value by $\hat{y}$. Subindices will be used when referring to more than one example in a dataset. Given an example $i$, the values of the $m$ input attributes are denoted by $x_{i,1},  x_{i,2}, \dots, x_{i,m}$. 
% Vectors (unidimensional arrays) are denoted in boldface and its elements with subindices, e.g., $\vect{v}= (v_1,v_2, \dots, v_n)$. Operations mixing arrays and scalar values will be allowed, specially in algorithms, as usual in the matrix arithmetic of many statistical computing languages. For instance, $\vect{v} + c$ means that the constant $c$ is added to all the elements in the vector $\vect{v}$.
%
Throughout the paper we will use several classifiers from Weka \cite{weka}. In this paper we are especially interested in using the techniques as they are, being able to use techniques that are, in principle, inattentive to the use of all the attributes, such as kernel methods, ensembles, etc. In particular, we will use SMO (a support vector machine), IBk (a k-nearest neighbour), J48 (a decision tree), Adaboost (an ensemble method with J48 decision trees) and Bagging (an example method with J48 decision trees). All of them will be used with their default parameters. 

Once this common setting for classification is set, we may wonder how models are created and deployed. In fact, models are usually learned under some contextual information but possibly deployed several times under changing conditions.
Reuse of learned models is of critical importance in the majority of knowledge-intensive application areas, particularly because the operating context can be expected to vary from training to deployment and we need to make the best decision according to that context \cite{SDM00,flach2003decision}. One kind of context is related to the way inputs (i.e., features) can vary from training to deployment. Among these changes, we can mention two important ones: attributes may not be available (missing values) or may have different test costs. Another type of context depends on how class distribution and misclassification costs affect the output variable. Note that these context changes may happen for each problem instance individually. For instance, in a medical domain, some tests may not be applicable to some patients (as can be contraindicated or risky), and other tests may be more or less expensive depending on the patient (her insurance policy). Also, for the output variable, a wrong diagnosis usually has asymmetric costs, as a false negative is usually worse (and economically more expensive in the long term) than a false positive. Again, these costs may be different for each example.

These two types of costs (test costs and misclassification costs) are highly intertwined. In fact, as Turney \cite{turney2000types} points out, we can only rationally determine whether it is worthwhile to pay the cost of test when we know the cost of misclassification errors. If the cost of misclassification errors is much greater than the cost of tests, then it is rational to purchase all tests that seem to have some predictive value. But if the cost of misclassification errors is much less than the cost of tests, then it is not rational to purchase any tests.

Let us define these types of cost formally:

\begin{definition}\label{def:M}
A misclassification cost function is any function $M:{\mathbb{Y}} \times {\mathbb{Y}} \rightarrow \mathbb{R}$ which compares elements in the output domain. For convenience, the first argument will be the estimated value, and the second argument the actual value. %, so its application is usually denoted by ${\ell}(\hat{y},y)$.
\end{definition}
As ${\mathbb{Y}}$ is a discrete set, typically we refer to $M$ as the misclassification cost {\em matrix}. We will assume that the diagonal of the matrix is zero (i.e., $\forall y \:: M(y,y) = 0$) and that the other elements of the matrix are greater than or equal to 0.

We can have a different matrix for each example, denoted by $M_i$. 
From above, we define the misclassification cost $MC$ of an example $i$ as $MC_i \triangleq M_i(\hat{y}_i,y_i)$. 
Only when the matrix is the same for all the examples, we can just calculate the average $MC$ as the Frobenius product between the confusion matrix for the whole dataset and the cost matrix, divided by $n$.

\begin{definition}\label{def:T}
The test cost vector is a real vector of size $m$, i.e., $(t_1, t_2, \dots, t_m)$, where $m$ is the number of attributes. 
The test cost function $T_j$ is any function as follows:
\begin{eqnarray*}
T_j(x) \triangleq \left\{   \begin{array}{l l}
                                             t_j    & \quad \text{if $x$ is not null}\\
                                             0      & \quad \text{otherwise}
                       \end{array} \right.																	
\end{eqnarray*}						
\end{definition}

We can have a different test cost function for each example and attribute, denoted by $T_{i,j}$. 
From above, we define the test cost $TC$ of an example $i$ as $TC_i \triangleq \sum_{j=1}^m T_{i,j}(x_{i,j})$. 
Only when $T_{i,j}$ are independent of the example $i$ we can just calculate the average $TC$ as the dot product between the use vector (how many times each attribute has been used for the dataset) and the test cost vector, divided by $n$.

We want to integrate both the misclassification cost and the test cost in one single measure of cost:

\begin{definition}\label{def:JC}
The {\em joint} cost for example $i$ is:
\begin{eqnarray*}
JC_i \triangleq \alpha \cdot MC_i + (1-\alpha) \cdot TC_i															
\end{eqnarray*}						
with $\alpha \in [0,1]$.
\end{definition}
The value $\alpha$ will be better explained later on, but clearly sets more relevance to misclassification or test costs. If $\alpha = 1$ only the misclassification cost matters, and if $\alpha = 0$  only the test cost matters.
$M$, $T$ and $\alpha$ configure the {\em cost context} or {\em operating condition}. With $m$ attributes and $c$ classes, there are $m+c(c-1)-1$ degrees of freedom (assuming the cost matrix has a zero diagonal).

\begin{example}\label{ex:1}
Consider the iris dataset \cite{UCIrep2013}, created by R.A. Fisher, which is composed of four attributes: $SL$, $SW$, $PL$ and $PW$ and three classes: {\em setosa}, {\em versicolour} and {\em virginica}. 

Assume that we have an example where the test cost vector is $(3, 2, 10, 5)$ and the misclassification cost matrix $M$ is defined as follows:

%\begin{table}
\begin{center}
{\center
\begin{tabular}{c|ccc}
            & setosa & versicolour & virginica \\ \hline
setosa      & 0      &         20  & 15        \\
versicolour & 5      &         0   & 15        \\
virginica   & 30     &         15  & 0         \\
\end{tabular}
%\caption{Bla.}\label{tab:levels}
}
\end{center}
%\end{table}

\noindent where columns represent the actual value and rows the predicted value. 
Consider also that we have three models to be applied to the same instance. Model 1 requires attributes $SL$ and $PL$ and predicts $virginica$, model 2 requires attributes $SL$ and $SW$ and predicts $setosa$, and model 3 requires all attributes and predicts $versicolour$. 
If the true label is $versicolour$, then we have $JC = MC + TC = 15 + (3 + 10) = 28$ for model 1, $JC = MC + TC = 20 + (3 + 2) = 25$ for model 2, and $JC = MC + TC = 0 + (3 + 2 + 10 + 5) = 20$ for model 3. 
\end{example}

In the previous example, model 3 is better than the other two for this example. Of course, in general, we need to make the decision of which model to use without knowing the actual label, and that will depend on the reliability of the models and the class frequencies. This is then a decision problem that can be solved by determining the model with lowest expected cost.

Interestingly, we may wonder what would happen if we removed attribute $PW$ for model 3. Even if we are told that model 3 was trained to work with that attribute, it is not difficult to guess what the model could do without it and still give a prediction. As we will see in more detail below, there are (at least) two ways to do it. First, we could set $PW$ to null (i.e., make it missing) and see what happens. Second, we could consider as range values for the attribute and get the most frequently predicted class. 
Clearly, none of these methods actually requires the attribute but allows us to use model 3.
Imagine that, by using any of these two methods, model 3 still predicts $versicolour$. Our cost would have been lowered down to $15$.

So, the question we want to address in this paper is not only what model to choose but also the subset of attributes that we will use (`buy'). How can we analyse this problem systematically? Can we use any kind of technique in machine learning, statistics and data mining, including ensembles, kernel methods, etc., where the tests costs are originally high.

%%%%%%%%%%%%%%%%%%%%%%%%%%%%%%%%%%%%%%%%%%%%%%%%%%%%%%%%%%%%%%%%%%
%%%%%%%%%%%%%%%%%%%%%%%%%%%%%%%%%%%%%%%%%%%%%%%%%%%%%%%%%%%%%%%%%%
\section{Reframing the model with missing values on purpose}\label{reframing}
%%%%%%%%%%%%%%%%%%%%%%%%%%%%%%%%%%%%%%%%%%%%%%%%%%%%%%%%%%%%%%%%%%
%%%%%%%%%%%%%%%%%%%%%%%%%%%%%%%%%%%%%%%%%%%%%%%%%%%%%%%%%%%%%%%%%%

There has been an extensive work in the past decades on how the performance of a predictive technique evolves with different feature subsets. This is the core of feature selection techniques. In fact, model performance can even be increased by using a subset of the original attributes. Also, if we think about costs, most works on minimising costs have taken this approach \cite{ling2004decision,zhang2005missing,lomax2013survey}.

However, we can also consider that the model has already been trained (with possibly all the attributes) and we may just want to apply the model with fewer available attributes, e.g., when missing values appear or when we cannot afford `buying' some of the tests included in the model. It is important to say that we consider models that may have been developed by experts or by automated predictive analysis tools, or both. Re-training can be a bad choice on many occasions: when we have an expert (human-made) model, when we are using ensembles or other techniques with high training costs, when the training data is no longer available, or when the cost context changes recurrently, even for each example.

What can we do instead of re-training? 
What we do is to {\em reframe} the original model to a situation with fewer attributes, a different {\em feature configuration}. But, how do models behave when we remove attributes from them?

First, we need to clarify how we can get predictions from a model that takes $m$ attributes when we only provide $m' < m$. There are two possible ways of reframing a model in order to do this: 

\begin{enumerate}
\item Setting the attribute to null. Many models can just work with missing values for test instances. However, on some occasions the model cannot take null values (e.g., logistic regression is usually one of these techniques). Nonetheless, it highly depends on the implementation of the technique (or the model).
\item \label{item-range} Instead, we can invent or negotiate over the attribute  \cite{bella2011using}. This means that if it is a nominal attribute, we can just ask the model to give a prediction for all the possible values for the attribute, get the predictions, and calculate the most frequently predicted class. If it is a numerical attribute, we can just use a sampling or discretisation and then behave similarly. If we have information about the attribute value distribution, we can also use it, as in missing value imputation.
\end{enumerate}

\noindent This second approach is more powerful (and related to missing value imputation and feature selection). 
In fact, on occasions, we may even realise that we get the same prediction for whatever value of the attribute (i.e., this is said to be a non-negotiable attribute in terms of \cite{bella2011using}) so we can clearly save the cost of getting the value for this attribute. 
However, for simplicity, we will work with the first way, as using a null value works for many DM/ML techniques and libraries, without further modification of our models. In our case, it just worked smoothly with Weka \cite{weka}.

\begin{figure}%[ht]
\centering
%\hspace{1cm}
\includegraphics[width=0.48\textwidth]{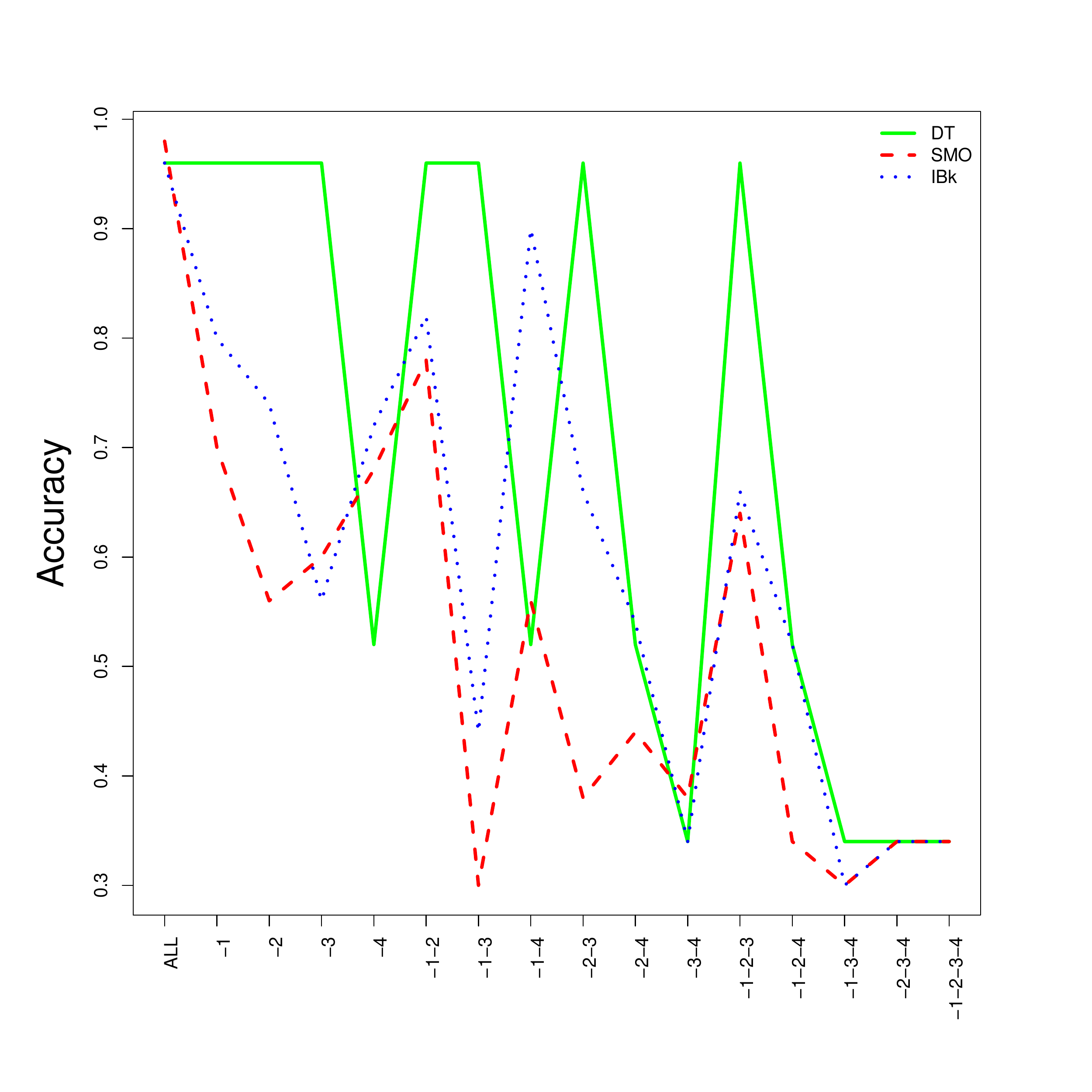} \hfill
\includegraphics[width=0.48\textwidth]{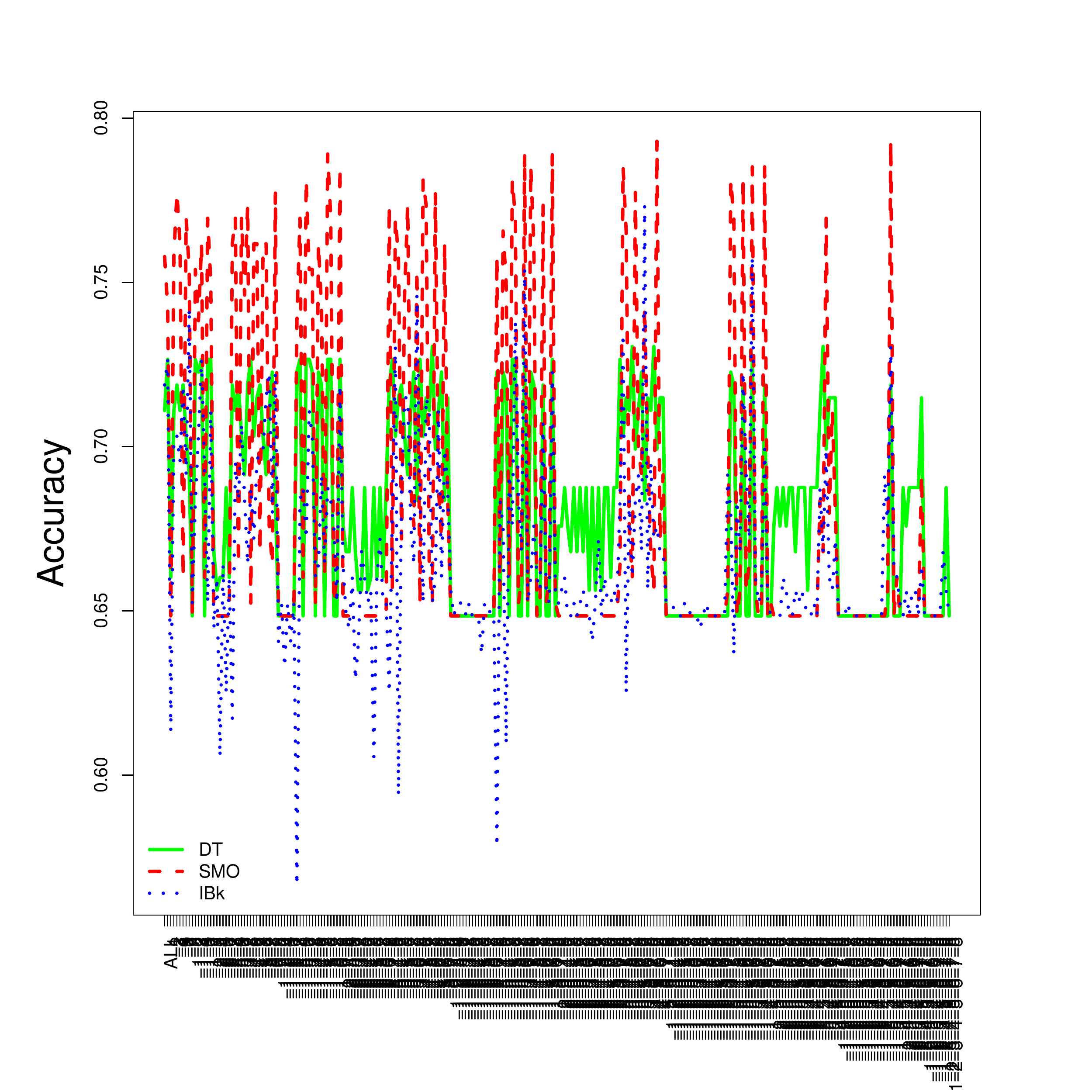} %\hspace{1cm}
\caption{Evolution of accuracy according to attribute selection for three models: decision trees, SMO and kNN. The models are learned over the whole training dataset using all the attributes. Then, some attributes are removed by setting a null value on them systematically, as shown on the $\xaxis$. Left: iris dataset (with $4$ attributes and hence $2^4=16$ combinations). Right: Pima Indian diabetes dataset (with $8$ attributes and hence $2^8=256$ combinations).}
\label{fig:accuracy}
\end{figure}

Once we know a simple procedure to reduce the attribute set, let us analyse how models behave.
Figure \ref{fig:accuracy} shows the evolution of accuracy\footnote{We show accuracy, but we could show other measures such as AUC or MSE \cite{PRL09,JMLR12}}, for all the possible subsets of the iris and the diabetes dataset (the subset lattice). Models are trained for 2/3 of the data and evaluated with the rest. We see many interesting things here. First, the general pattern is to get more accuracy as more attributes are used. But, obviously, some attributes are more important than others, leading to a sawtooth picture. Second, and more interestingly, the minimum is found at the majority-class classifier, i.e., if we are not given any information about any attribute, the best thing that we can do is to predict the majority class (or the class with lowest expected loss if misclassification costs are taken into account). Third, now surprisingly, we see that for some models and problems (Figure \ref{fig:accuracy}, right), the maximum is not obtained with all the attributes. In fact, it is obtained at several other places, one of them with six attributes removed (of the possible eight).

We can show the specific values of $MC$, $TC$ and the aggregate $JC$ for a given context of $M$, $T$ and $\alpha$. 
We will first consider a `uniform' operating context:

\begin{definition}\label{def:uniform}
The uniform operating context $\theta_U$ is defined by a uniform test cost vector $(1/m, 1/m, \dots, 1/m)$ and a uniform misclassification cost matrix 
$\forall y_1,y_2 \:: M(y_1,y_2) = c/(c-1)$ if $y_1 \neq y_2$ and 0 otherwise. Also, $\alpha=0.5$.
\end{definition}

The parameters of this context have the property that given a problem whose classes are perfectly balanced, the expected $MC$ of a random classifier is $1$ and the expected $TC$ of a classifier using all the attributes is $1$. As a consequence, $JC=1$. For this context, any model with $JC > 1$ is clearly a model to be discarded. In fact, as a random classifier does not need to use any of the attributes, any $JC > 0.5$ is also discardable for this context. It is easy to see that if $T$ is the uniform test cost vector and $M$ is the uniform misclassification cost matrix, we have that $\sum T = 1$ and $\sum M = c^2$. This property will be known as a context being normalised.

Figure \ref{fig:allcostsU} shows the evolution of $MC$, $TC$ and the aggregate $JC$ for the uniform context described above. We see that the in formation shown is very similar to that evolution of Figure \ref{fig:accuracy}. However, for other operating contexts, things might be different. Let us see this.

\begin{figure}%[ht]
\centering
%\hspace{1cm}
\includegraphics[width=0.48\textwidth]{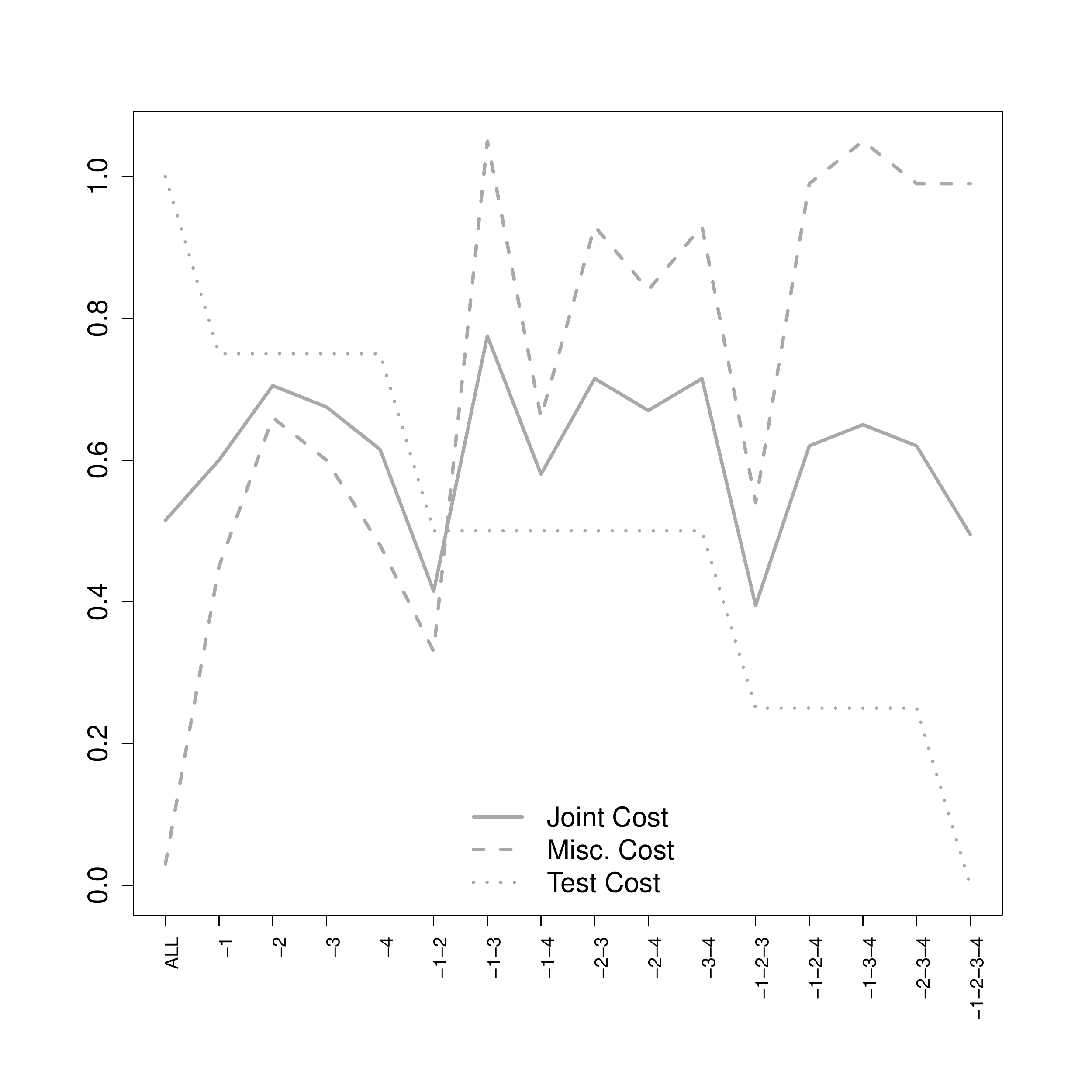} \hfill
\includegraphics[width=0.48\textwidth]{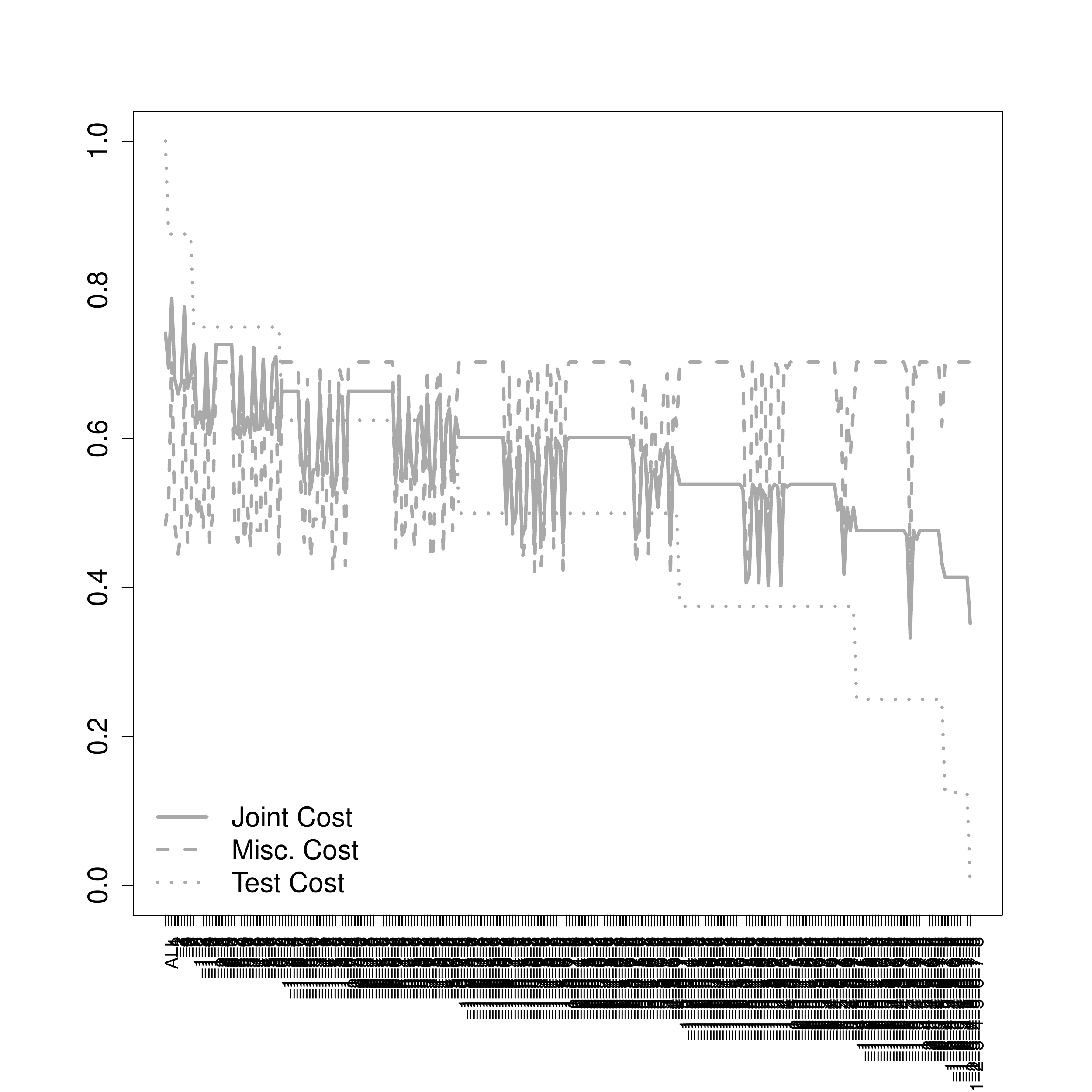} %\hspace{1cm}
\caption{Evolution of $MC$, $TC$ and $JC$ according to attribute selection for a SMO (SVM) model using the uniform context (see definition \ref{def:uniform}). Left: iris dataset. The configuration which minimises the $JC$ is given by the use of only attribute 4 (removing -1-2-3) Right: Pima Indian diabetes dataset. The configuration that minimises the $TC$ is given by only two attributes (removing six).}
\label{fig:allcostsU}
\end{figure}

%Now let us consider another 
The new (non-uniform) operating context is defined as follows for the problems ``iris'' and ``Pima Indian diabetes''.

The operating context $\theta_1$ for ``iris'' is just the one in example \ref{ex:1}. The operating context $\theta_2$ for ``Pima Indian diabetes'' is defined as a test cost vector is $(2, 50, 5, 5, 20, 3, 10, 1)$, which means that the most expensive tests correspond to `plasma glucose concentration', `2hour serum insulin' and `diabetes pedigree function'. The misclassification cost matrix $M$ is defined as follows:

%\begin{table}
\begin{center}
{\center
\begin{tabular}{c|cc}
             & negative (0)  & positive (1) \\ \hline
negative (0) & 0             &         200          \\
positive (1) & 50            &         0           \\
\end{tabular}
%\caption{Bla.}\label{tab:levels}
}
\end{center}
%\end{table}
\noindent where columns represent the actual value and rows the predicted value. The value of $\alpha$ is 0.5.

With these operating contexts, Figure \ref{fig:allcosts0102} shows the same plots as Figure \ref{fig:allcostsU}, with a different result.

\begin{figure}%[ht]
\centering
%\hspace{1cm}
\includegraphics[width=0.48\textwidth]{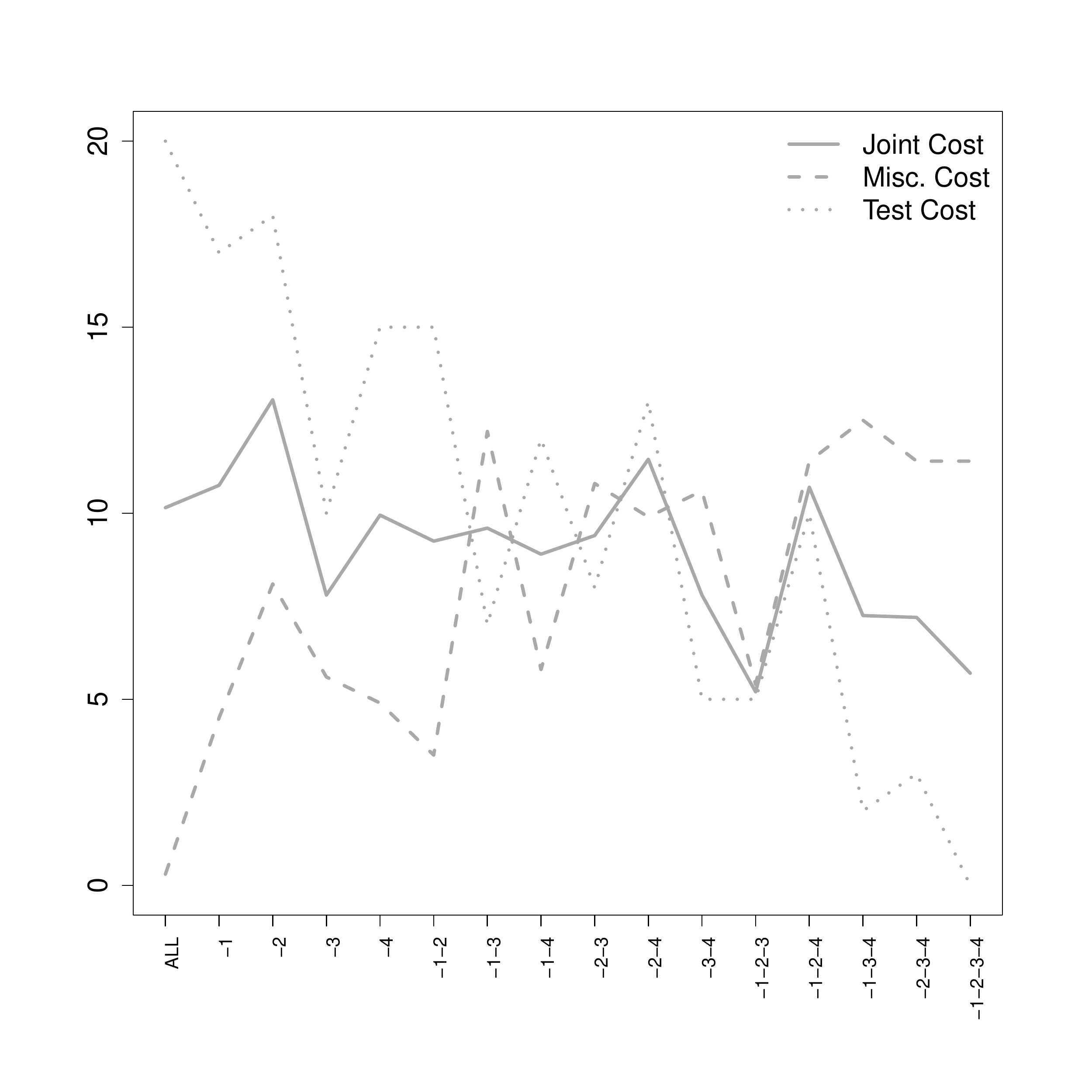} \hfill
\includegraphics[width=0.48\textwidth]{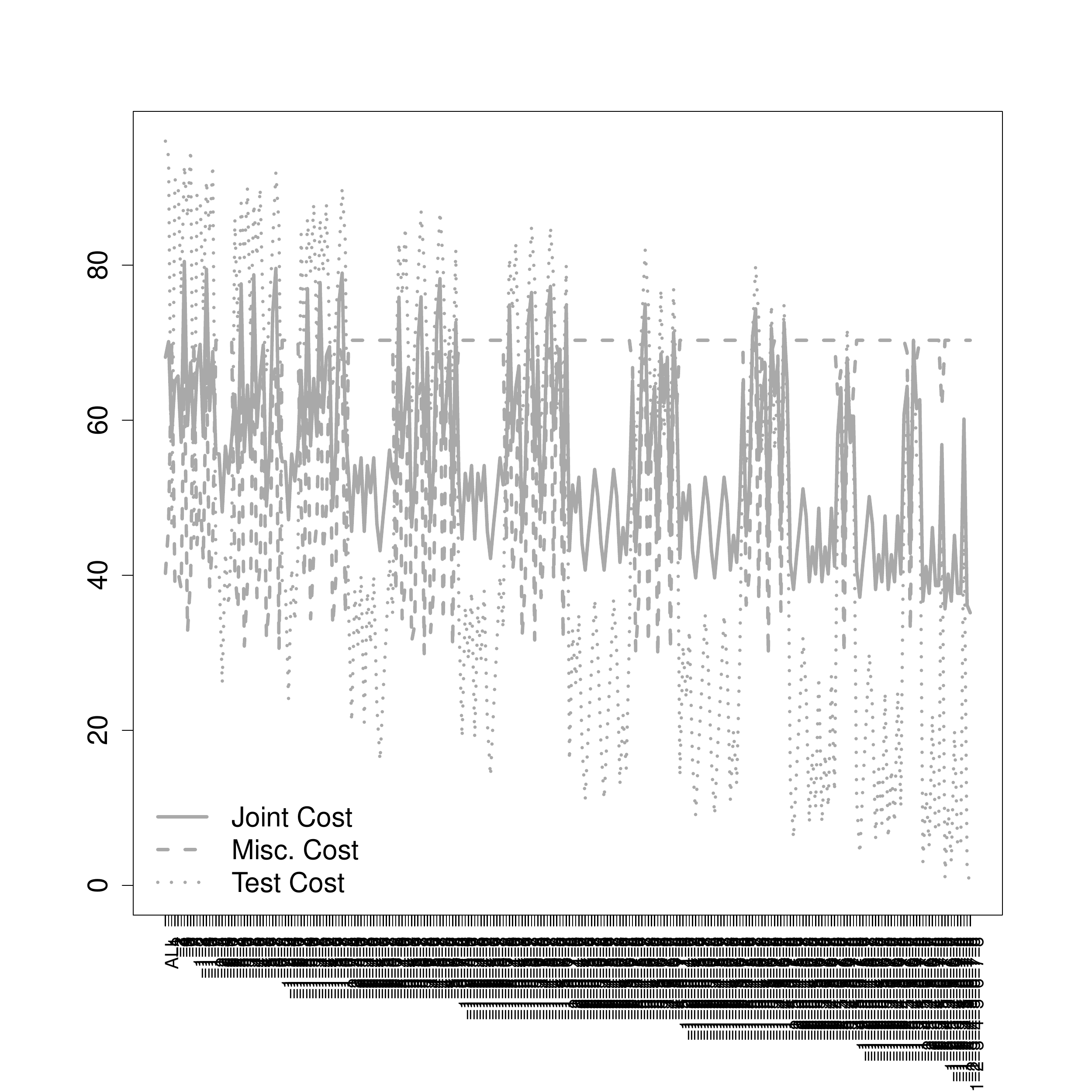} %\hspace{1cm}
\caption{Evolution of $MC$, $TC$ and $JC$ according to attribute selection for a SMO (SVM) model. Left: iris dataset using context $\theta_1$. The configuration that minimises the $TC$ is given by attribute 4. Right: Pima Indian diabetes dataset using context $\theta_2$. The configuration that minimises the TC is given by removing all the attributes.}
\label{fig:allcosts0102}
\end{figure}

%%%%%%%%%%%%%%%%%%%%%%%%%%%%%%%%%%%%%%%%%%%%%%%%%%%%%%%%%%%%%%%%%%
%%%%%%%%%%%%%%%%%%%%%%%%%%%%%%%%%%%%%%%%%%%%%%%%%%%%%%%%%%%%%%%%%%
\section{The MC/TC trade-off: \JROC plots}\label{tradeoff}
%%%%%%%%%%%%%%%%%%%%%%%%%%%%%%%%%%%%%%%%%%%%%%%%%%%%%%%%%%%%%%%%%%
%%%%%%%%%%%%%%%%%%%%%%%%%%%%%%%%%%%%%%%%%%%%%%%%%%%%%%%%%%%%%%%%%%

The plots seen in the previous section are very informative for a given operating context. If the plots are drawn on a validation set, we will just choose the model and feature configuration that minimises the $JC$. However, there are some problems with the previous plots: if we have several models, the plot gets too crowded. Also, the {\em curves} are usually too sawtooth. Finally, we need to change the curves whenever we change the operating context.

While some of the above problems are difficult to solve completely, most especially because we have $m+c(c-1)-1$ degrees of freedom, we can see a more convenient alternative that minimises these problems. The alternative is based on a graphical visualisation of the MC/TC trade-off that we call {\em \JROC plots}.

\begin{definition}\label{def:jroc}
A \JROC plot shows the test cost ($TC$) on the $\xaxis$ and misclassification cost ($MC$) on the $\yaxis$.
\end{definition}

Figure \ref{fig:JROCU} shows $\JROC$ plots for iris and Pima Indian diabetes. 
For iris, as it has four attributes, we see $2^4 \times 3$ points, $2^4$ for each model.
For diabetes, as it has eight attributes, we see $2^8 \times 3$ points, $2^8$ fore each model.
Those models and configurations which go closer to the bottom left corner are better than those that are placed on the top right area of the plot. 
There is always a point with 0 $TC$ and a usually high misclassification cost, frequently matching the majority class model. 
However, as mentioned earlier on, the minimum $MC$ is not always achieved with maximum $TC$. 
In this particular case, the minimum $MC$ for diabetes is obtained with a $TC$ of 0.22.

\begin{figure}%[ht]
\centering
%\hspace{1cm}
\includegraphics[width=0.48\textwidth]{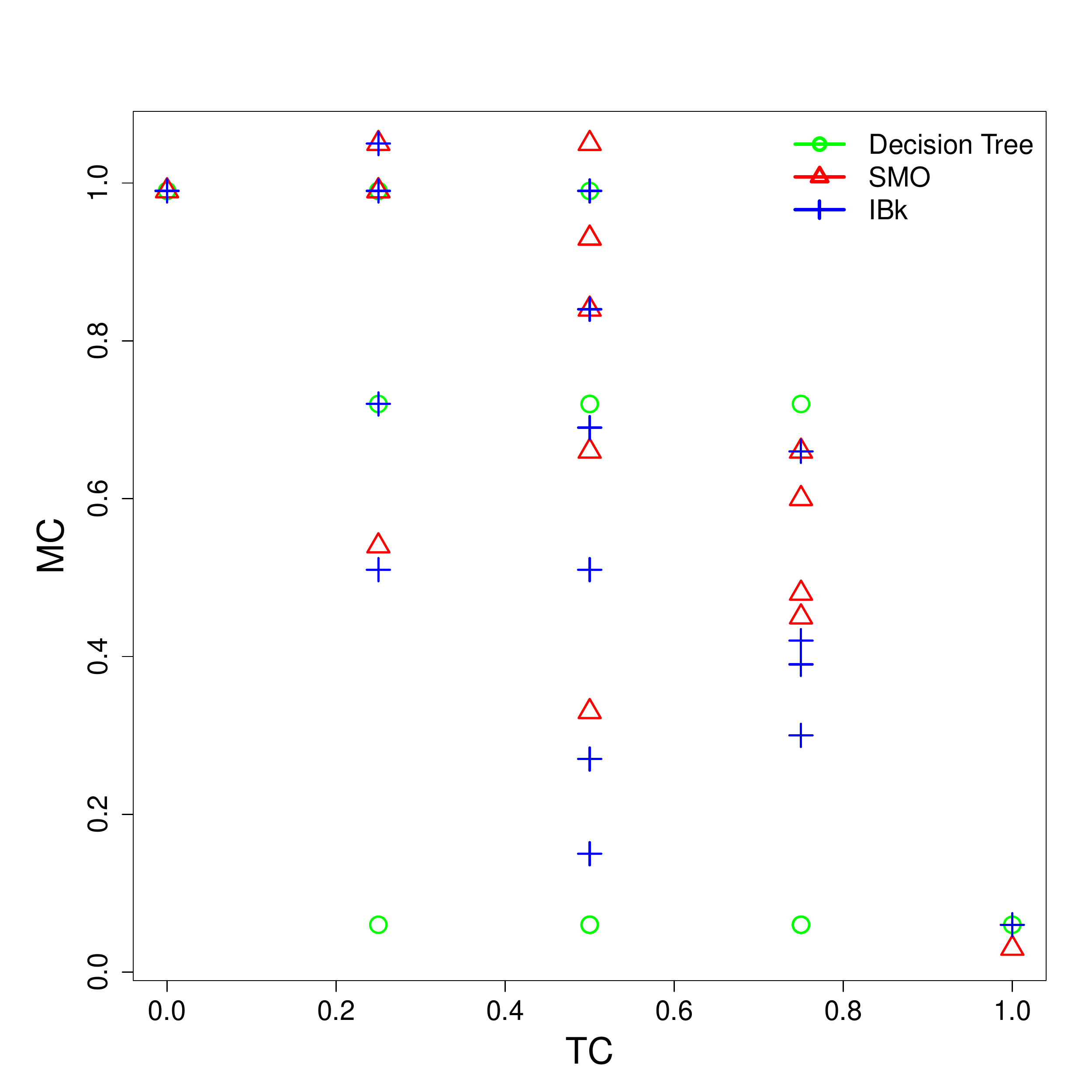} \hfill
\includegraphics[width=0.48\textwidth]{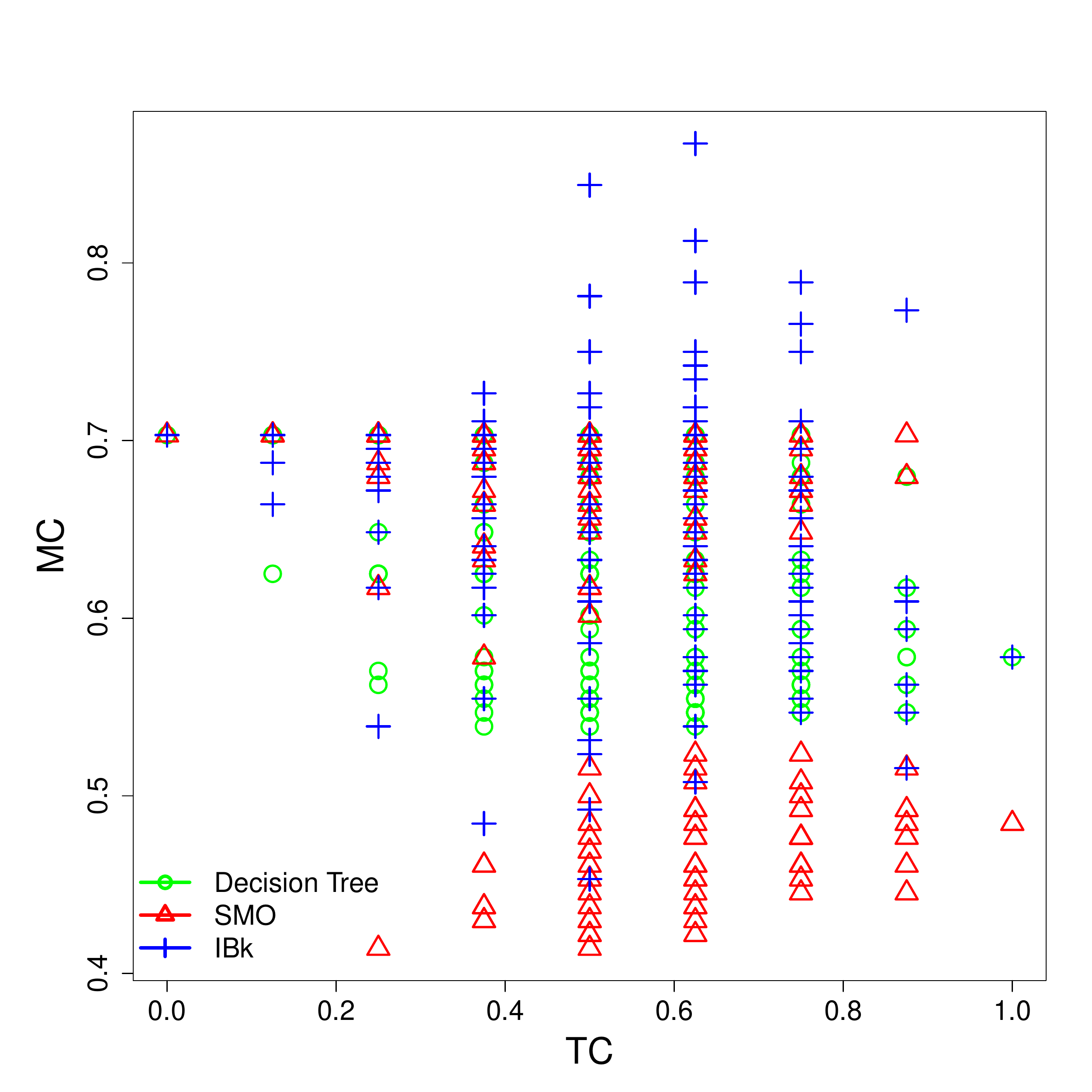} %\hspace{1cm}
\caption{\JROC plots for the three models: decision trees, SMO and IBk. We use the uniform operating context $\theta_U$. We show test cost ($TC$) on the \xaxis and misclassification cost ($MC$) on the \yaxis. Left: iris dataset. Right: Pima Indian diabetes dataset. For iris we see that decision trees and kNN (IBk) perform better, as the points which are most on the bottom left are of these models. However, for diabetes, it seems that SMO and kNN get closer to the desired bottom left corner.}
\label{fig:JROCU}
\end{figure}

Figure \ref{fig:JROC0102} shows a similar plot with different cost contexts. Here, we also see how the points are now located in different places. Even though the classifiers are the same, the distribution of the points is very different.

\begin{figure}%[ht]
\centering
%\hspace{1cm}
\includegraphics[width=0.48\textwidth]{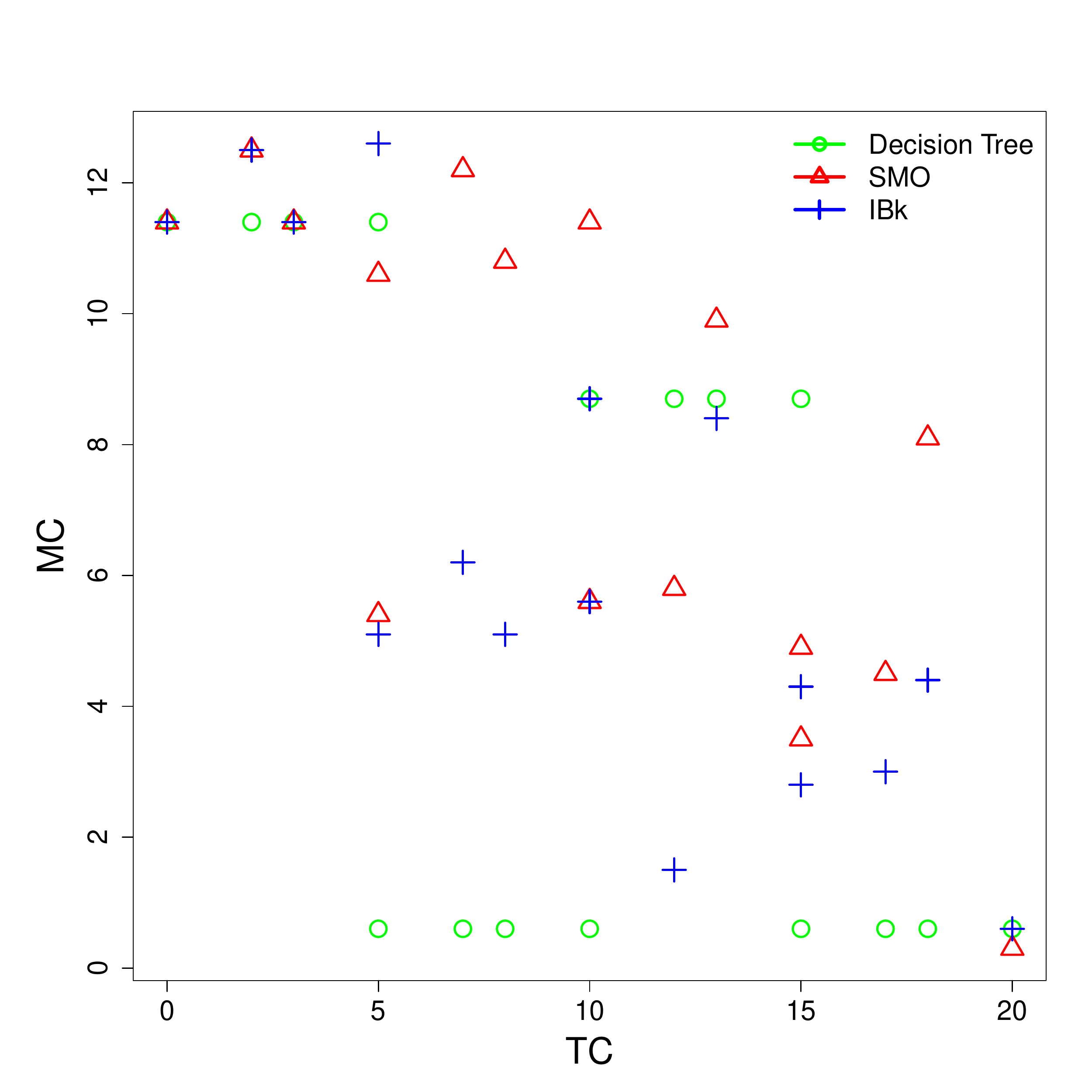} \hfill
\includegraphics[width=0.48\textwidth]{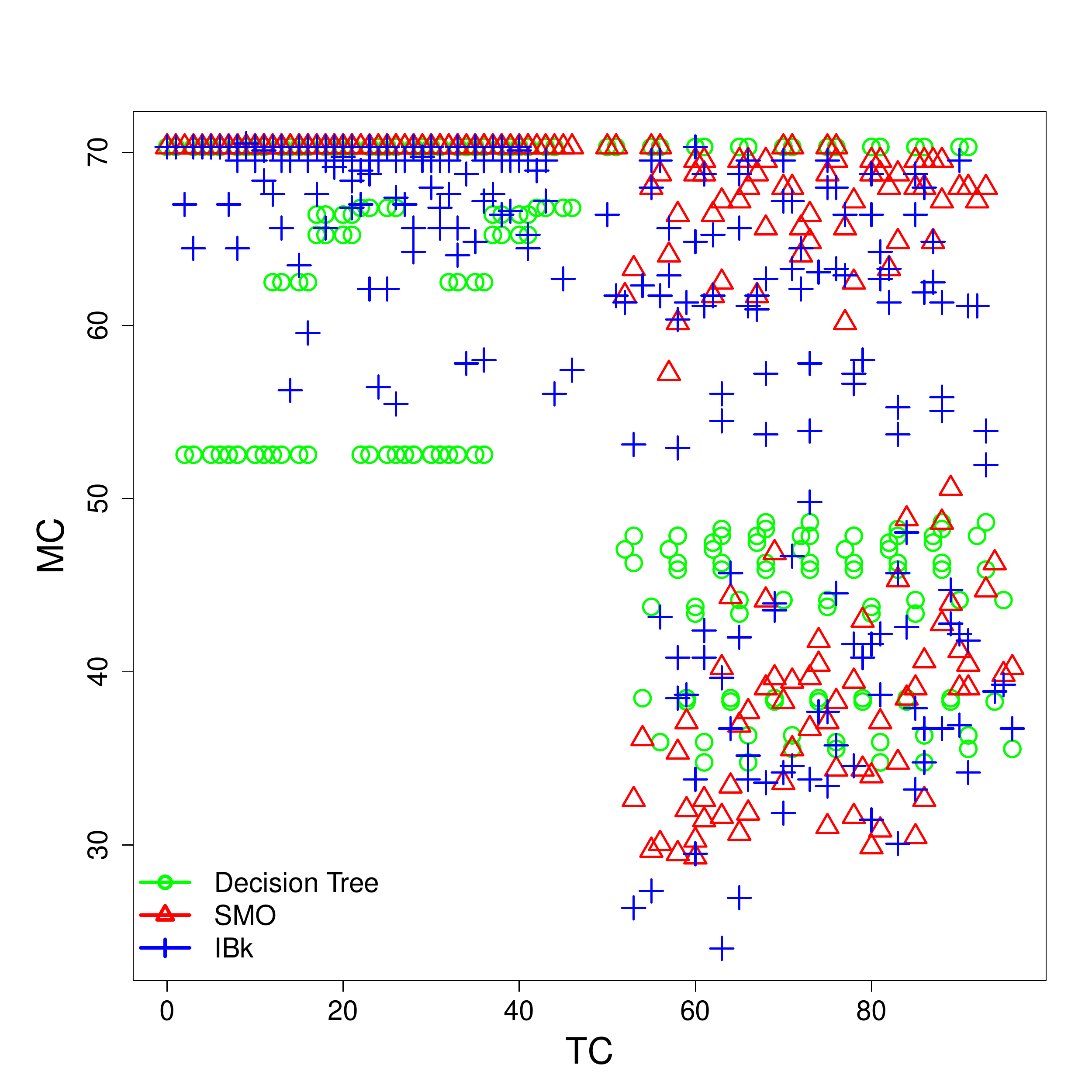} %\hspace{1cm}
\caption{\JROC plots for the three models: decision trees, SMO and IBk. We show test cost ($TC$) on the \xaxis and misclassification cost ($MC$) on the \yaxis. Left: iris dataset with the operating context $\theta_1$. Right: Pima Indian diabetes dataset with the operating context $\theta_2$. Compare to Figure \ref{fig:JROCU}, which uses a different cost context.}
\label{fig:JROC0102}
\end{figure}

Intentionally, we have not shown $\alpha$ on the plots, even though, according to definition \ref{def:JC}, we cannot calculate $JC$ unless this value is fixed. The following lemma shows that the same plot can be used to calculate $JC$ for any value of $\alpha$, with the notion of cost {\em isometrics}.

\begin{proposition}
Given a value of $\alpha$ the points which are connected by a line with slope $\frac{1-\alpha}{-\alpha}$ have the same $JC$.
\end{proposition}
\begin{proof}
From definition \ref{def:JC} we have that $JC = \alpha \cdot MC + (1-\alpha) \cdot TC$. 
Consequently, two points $a$ and $b$ have the same $JC$ iff
$\alpha \cdot MC_a + (1-\alpha) \cdot TC_a = \alpha \cdot MC_b + (1-\alpha) \cdot TC_b$. 
Operating with this equation, we get:
\begin{eqnarray*}
\alpha \cdot MC_a + (1-\alpha) \cdot TC_a & =  & \alpha \cdot MC_b + (1-\alpha) \cdot TC_b \\ 
MC_a + \frac{1-\alpha}{\alpha} \cdot TC_a & = &  MC_b + \frac{1-\alpha}{\alpha} \cdot TC_b \\ 
MC_a - MC_b & = & \frac{1-\alpha}{\alpha} \cdot (TC_b - TC_a)  \\ 
\frac{MC_a - MC_b}{TC_a - TC_b} & = & \frac{1-\alpha}{-\alpha}
\end{eqnarray*}
As the last expression is the change in $y$ divided by the change in $x$, the expression $\frac{1-\alpha}{\alpha}$ is the slope of this line.
\end{proof}

If $\alpha = 1$  only the misclassification cost matters and the slope is 0, and if $\alpha = 0$  only the test cost matters and the slope is infinite. It is clear that $\alpha$ only represents one of the $m+c(c-1)-1$ degrees of freedom, but it is able to consider the most important one: the relative relevance between misclassification and test costs.

As in classical ROC analysis, if we slide an isometric line given by a value of $\alpha$ from the point $(0,0)$ in the opposite direction (towards the top-right part of the plot), we will eventually find one  point (or more) on the plot. This is the best point according to the operating condition.

Figure \ref{fig:isometrics} shows three different isometrics given by operating conditions $\alpha=0.03$, $\alpha=0.5$  and $\alpha=0.9$ and where they touch on the cloud of points. As we can see, different feature configurations and models are chosen for each operating condition. 

\begin{figure}%[ht]
\centering
%\hspace{1cm}
\includegraphics[width=0.48\textwidth]{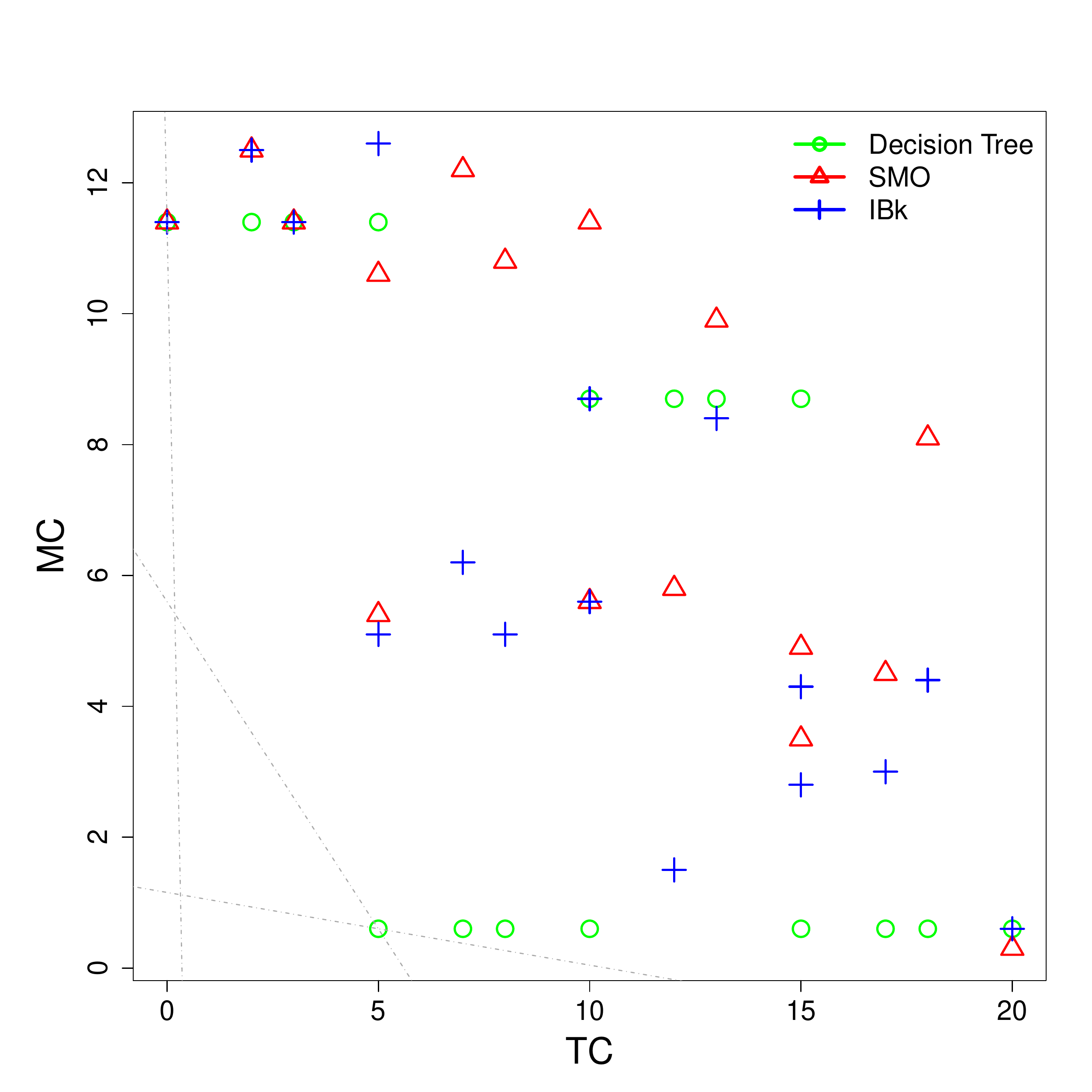} \hfill
\includegraphics[width=0.48\textwidth]{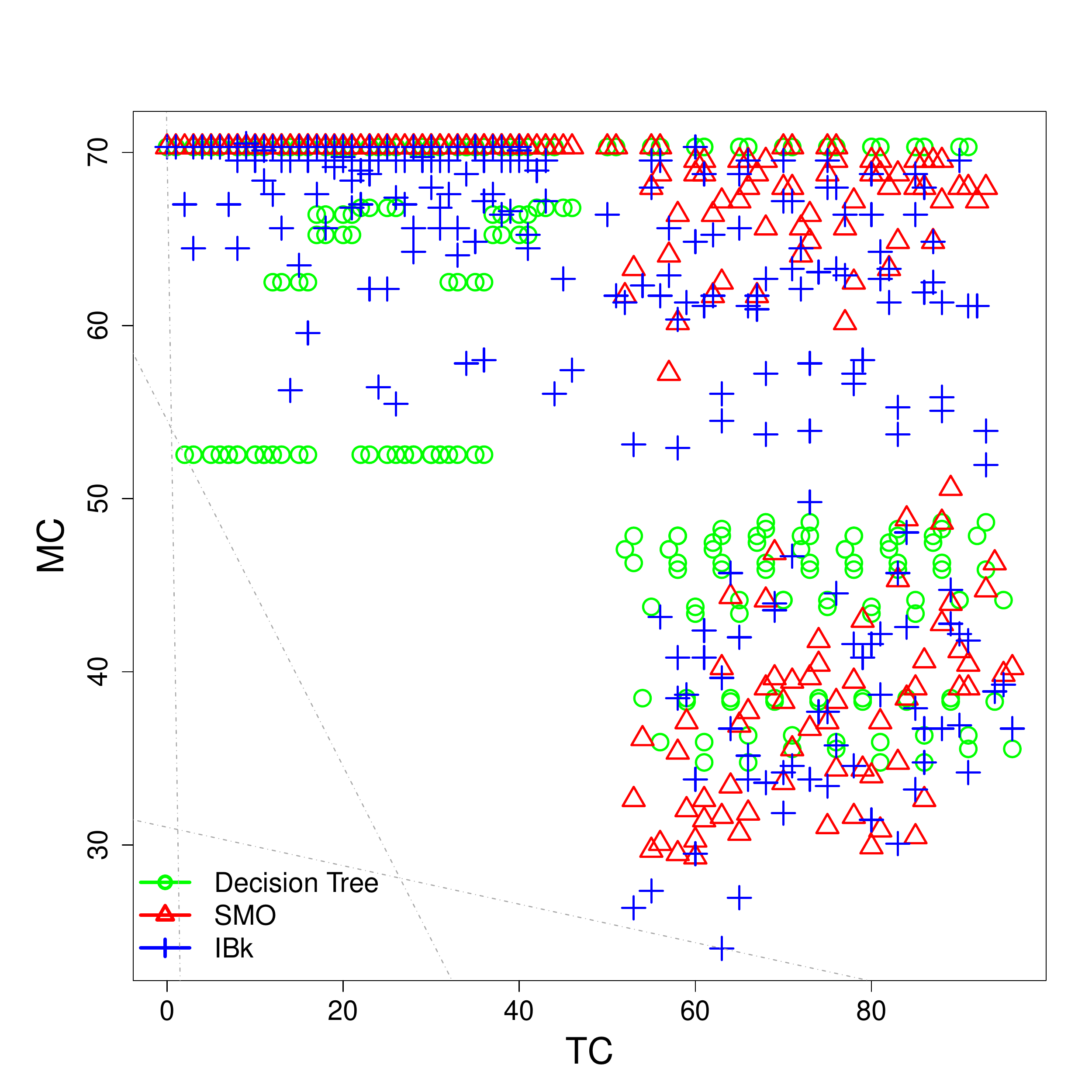} %\hspace{1cm}
\caption{We show the same plots as \ref{fig:JROC0102} but now we show isometrics for operating condition $\alpha=0.03$, $\alpha=0.5$  and $\alpha=0.9$. Left: iris dataset with the operating context $\theta_1$. Right: Pima Indian diabetes dataset with the operating context $\theta_2$.}
\label{fig:isometrics}
\end{figure}

Finally, if we consider all possible values of $\alpha \in [0,1]$ we see that some points are never chosen. This is exactly the notion of convex hull:

\begin{definition}\label{def:jroc-hull}
A \JROC convex hull of a model is the convex hull of the set of points on the \JROC space that are defined using all the attribute subsets.
\end{definition}

\begin{figure}%[ht]
\centering
%\hspace{1cm}
\includegraphics[width=0.48\textwidth]{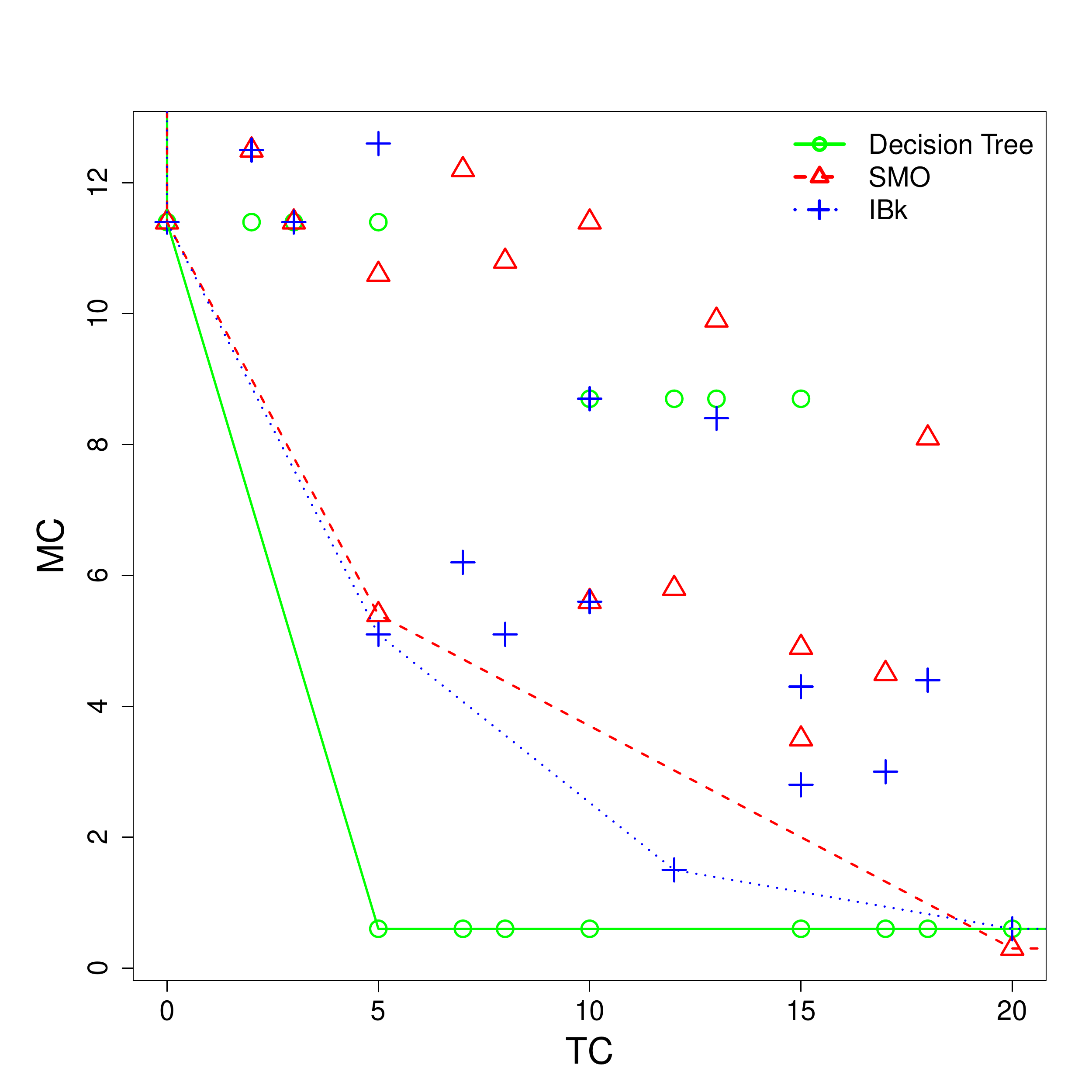} \hfill
\includegraphics[width=0.48\textwidth]{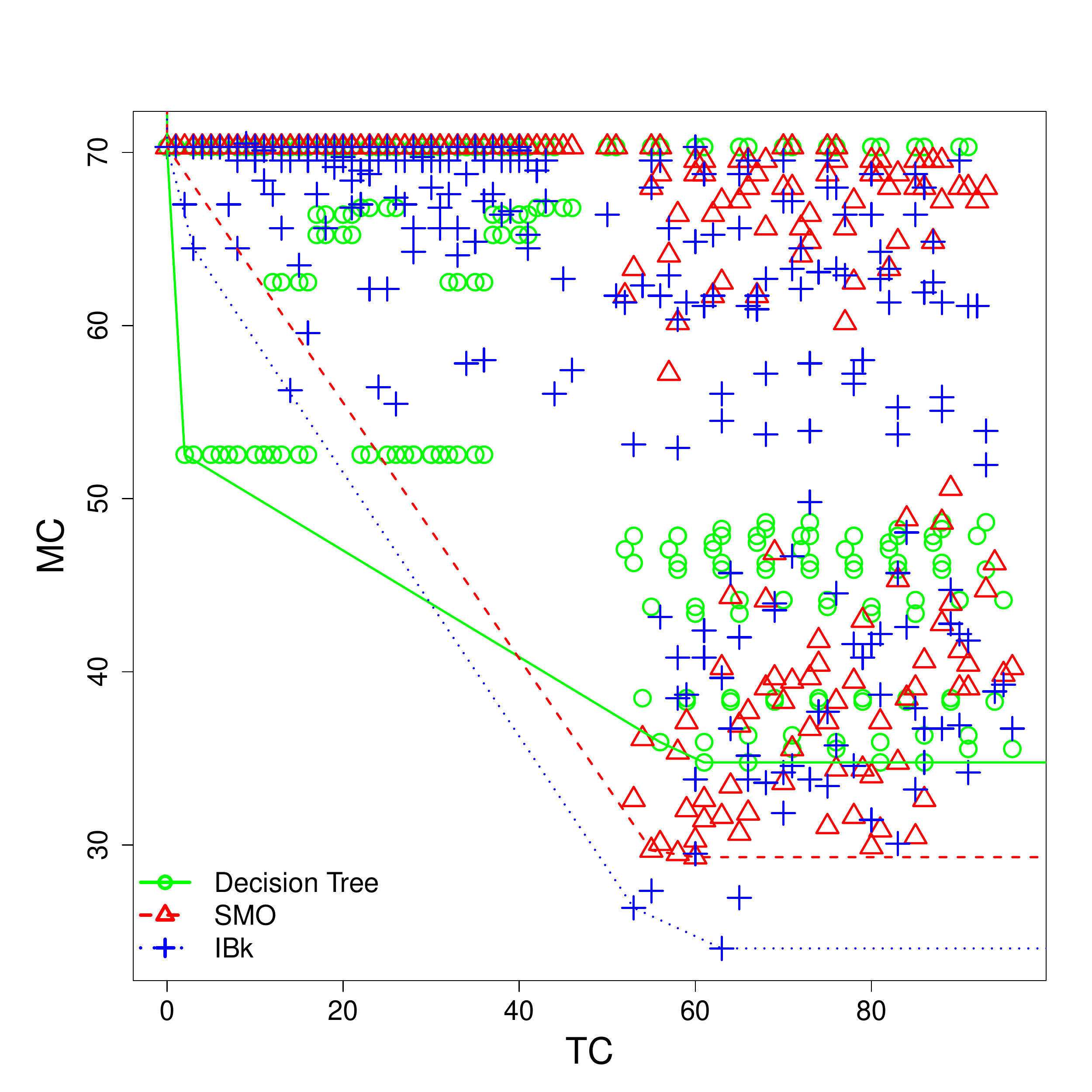} %\hspace{1cm}
\caption{We show the same plots as \ref{fig:JROC0102} but now the convex hulls for the three models are also drawn. Left: iris dataset with the operating context $\theta_1$. Right: Pima Indian diabetes dataset with the operating context $\theta_2$.}
\label{fig:hull}
\end{figure}

Figure \ref{fig:hull} shows the convex hull for each of the three models. We can also see the regions of dominance. For diabetes, IBk dominates for high values of $\alpha$, while the decision tree dominates for low values of $\alpha$.

From here we can calculate the regions of dominance for $\alpha$ and choose the best model accordingly, in the very same way as in ROC analysis.

%%%%%%%%%%%%%%%%%%%%%%%%%%%%%%%%%%%%%%%%%%%%%%%%%%%%%%%%%%%%%%%%%%
%%%%%%%%%%%%%%%%%%%%%%%%%%%%%%%%%%%%%%%%%%%%%%%%%%%%%%%%%%%%%%%%%%
\section{Approximating the \JROC hull}\label{hull}
%%%%%%%%%%%%%%%%%%%%%%%%%%%%%%%%%%%%%%%%%%%%%%%%%%%%%%%%%%%%%%%%%%
%%%%%%%%%%%%%%%%%%%%%%%%%%%%%%%%%%%%%%%%%%%%%%%%%%%%%%%%%%%%%%%%%%

The previous procedure allows us to determine the best model and configuration given the operating condition. We only need to calculate where all the points lie, compute the convex hull and find the one that corresponds for each possible $\alpha$ in application time. 
While this looks easy to do, there is one big issue. As the number of attributes increase, the number of points for a model grows exponentially: a lattice for $m$ attributes has $2^m$ nodes. For instance, for a model with 16 attributes, we would have $2^{16} = 65536$ points. Navigating the complete lattice of attribute subsets, and calculating their expected $TC$ and $MC$ would be infeasible. So we need to explore some ways to reduce the number of configurations that are evaluated, while still having a good approximation of the \JROC hulls in order to do the correct decisions and get the optimal cost.

We will consider how to reduce the number of configurations from an exponential growth ($O(2^m)$), given by a {\em full} method, to a quadratic growth ($O(m^2)$). We consider four possible\footnote{There would also be the {\em forward} versions as well. We rule these possibilities out here for the simplicity of exposition, and also because we think that the results would be similar, but they could also be considered in practice.} methods:

\begin{itemize}
\item Backward $MC$-guided (BMC): we start with one case with the $m$ attributes, we evaluate with the $m$ cases removing one attribute, and choose the best one in terms of $MC$, then we evaluate the $m-1$ cases removing one attribute from the previous one, and so on. This leads to exactly $1 + (m) + (m-1) + (m-2) + \dots + 1$ $= m(m+1) / 2 + 1$, which has an order of $O(m^2)$.
\item Backward $TC$-guided (BTC): as BMC but using $TC$ instead. It has the same order and number of points.
\item Backward $JC$-guided (BJC): as BMC but using $JC$ instead. It has the same order and number of points.
\item Monte Carlo (RND): a random sample over the lattice. In order to make comparison fair, we will also consider the same number of elements.
\end{itemize}

\noindent It is easy to show that if the misclassification cost matrix is uniform, then BTC and BJC are equivalent. If the test cost vector then BMC and BJC are equivalent. If both the misclassification cost matrix and the cost vector are uniform (i.e.,  the uniform operating context $\theta_U$) then BMC, BTC and BJC are equivalent.

Figure \ref{fig:bmc} shows the results for the BMC method for our two datasets and operating contexts $\theta_1$ and $\theta_2$. 

\begin{figure}%[ht]
\centering
%\hspace{1cm}
\includegraphics[width=0.48\textwidth]{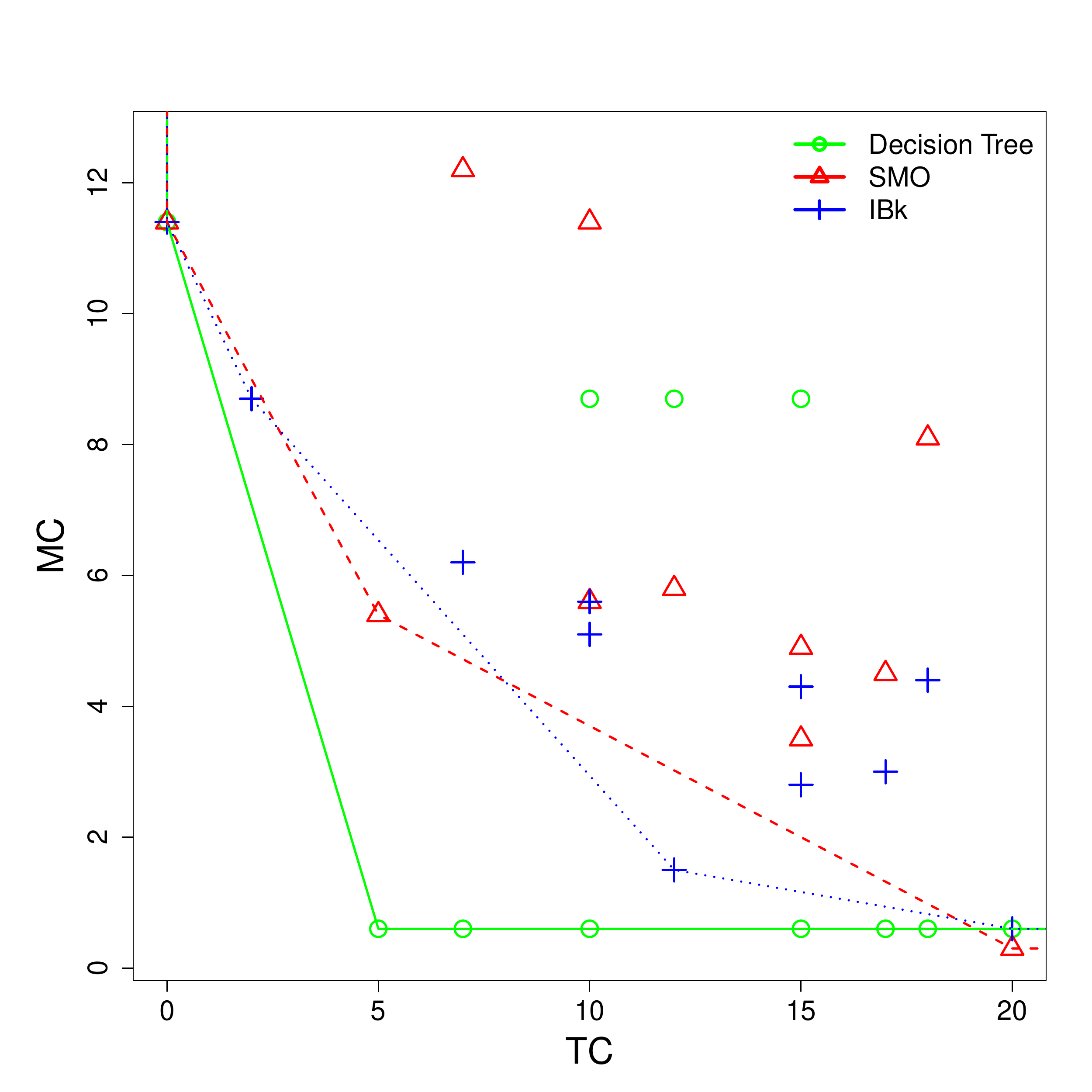} \hfill
\includegraphics[width=0.48\textwidth]{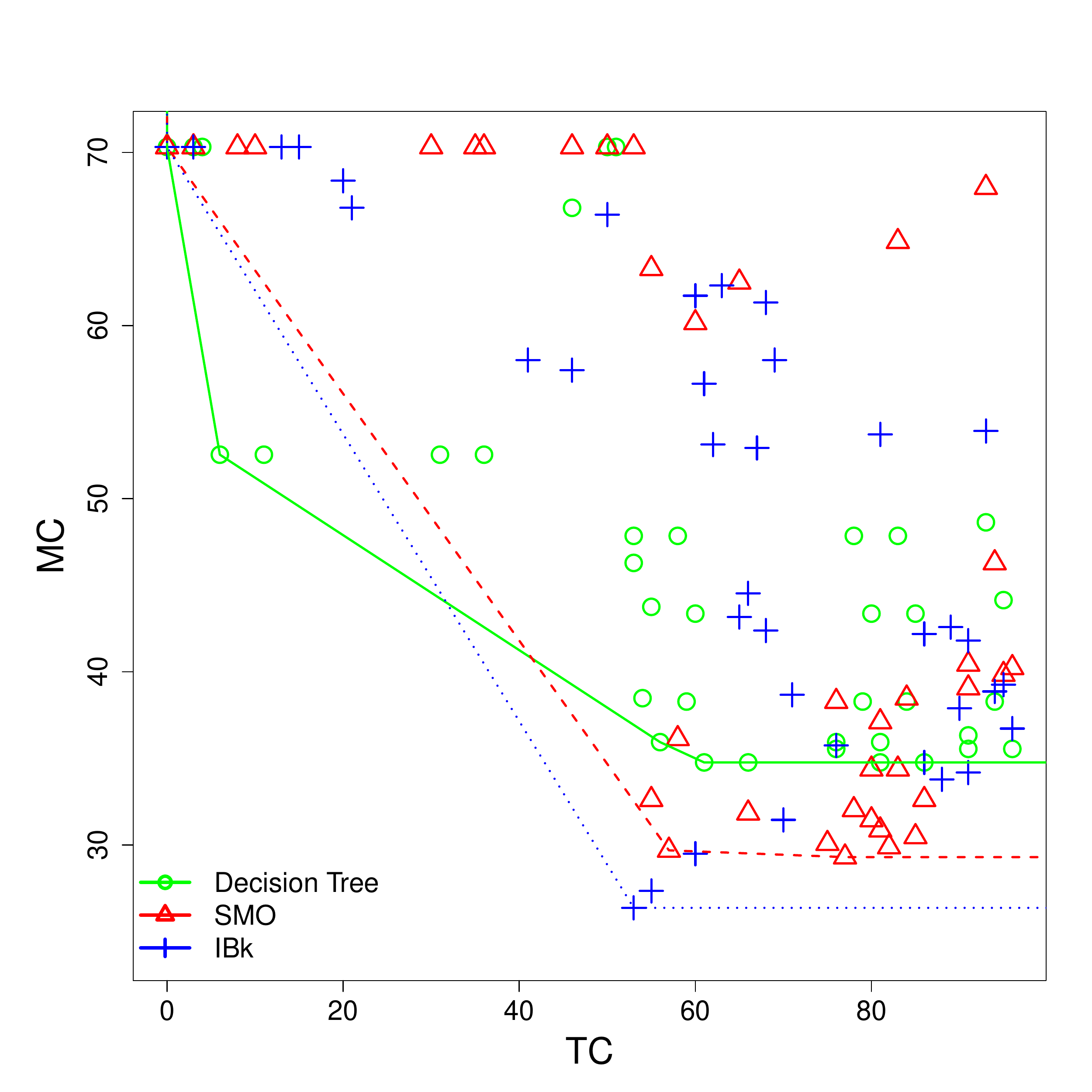} %\hspace{1cm}
\caption{We show the points and the hull of the combinations selected by the method BMC. There are $m(m+1) / 2 + 1$ points for each model instead of $2^m$. Compare with Figure \ref{fig:hull}. Left: iris dataset with the operating context $\theta_1$. Right: Pima Indian diabetes dataset with the operating context $\theta_2$.}
\label{fig:bmc}
\end{figure}

Similarly, we have the results for the BTC, BJC and RND methods on figures \ref{fig:btc}, \ref{fig:bjc} and \ref{fig:rnd}.
If we compare with Figure \ref{fig:hull}, we see that the hulls are much worse for BTC, and notably worse for BJC and RND. 
However, we see that the results for BMC are good, almost identical to Figure \ref{fig:hull}. 
Does this observation hold in general? This (and other things) are explored in the experiments.

\begin{figure}%[ht]
\centering
%\hspace{1cm}
\includegraphics[width=0.48\textwidth]{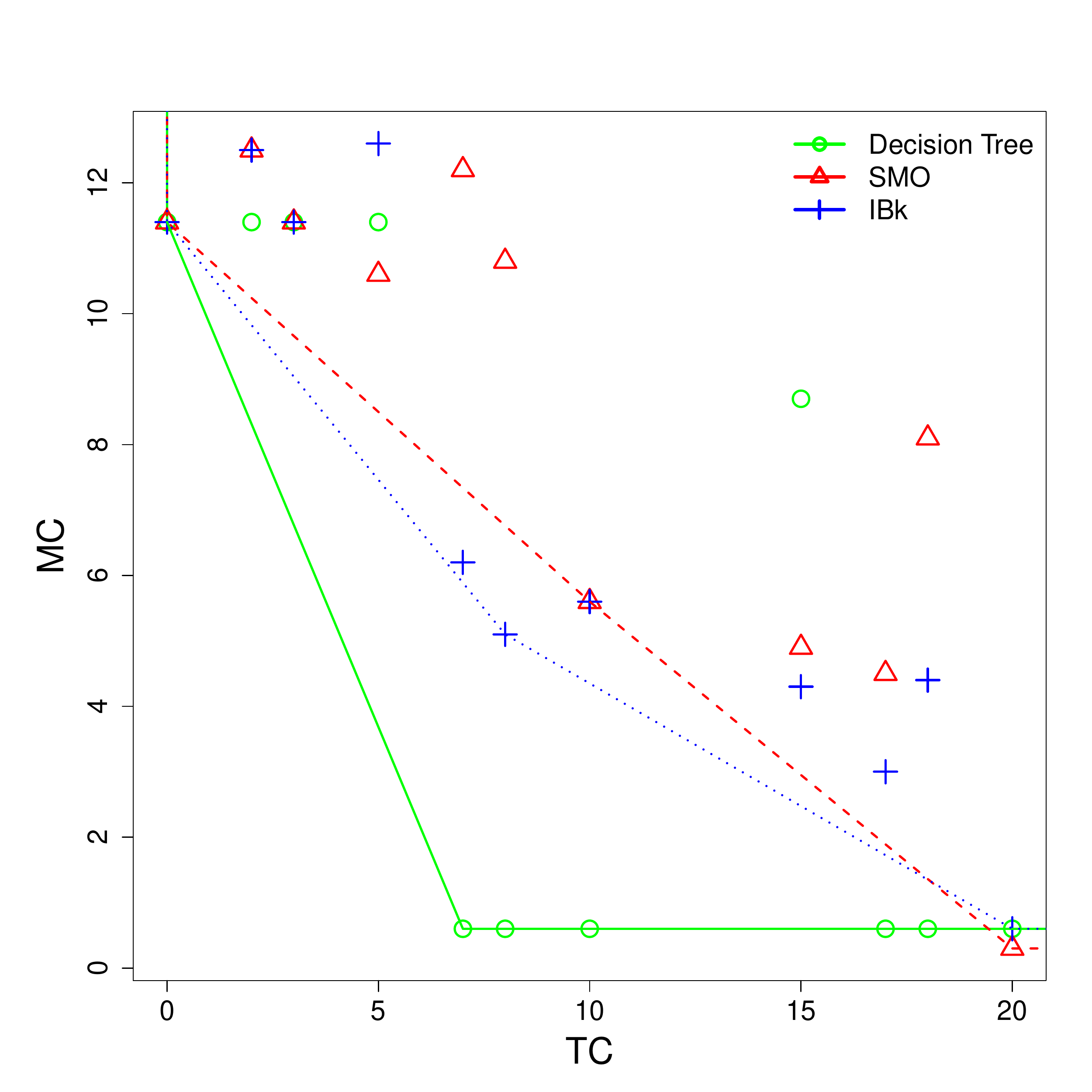} \hfill
\includegraphics[width=0.48\textwidth]{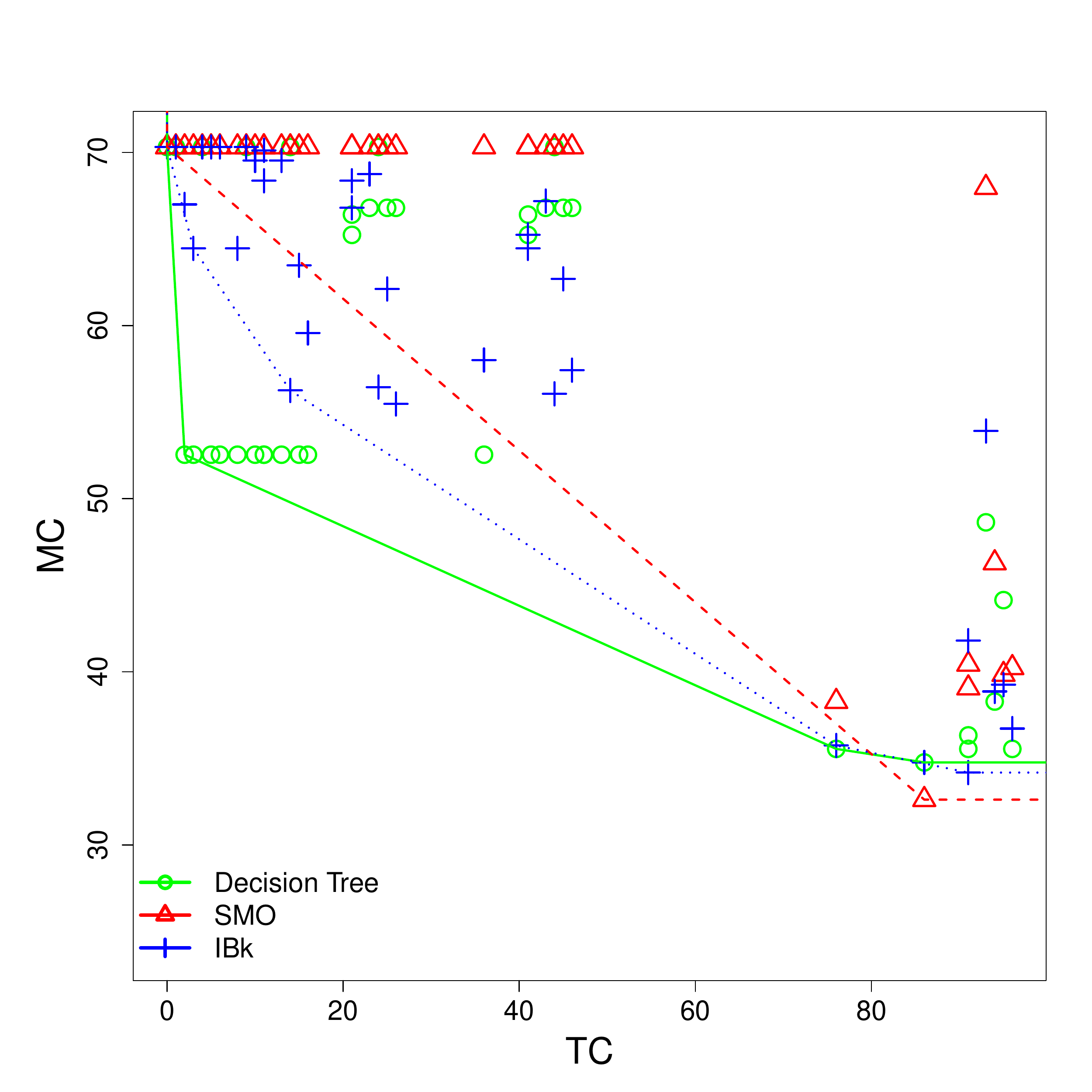} %\hspace{1cm}
\caption{We show the points and the hull of the combinations selected by the method BTC. There are $m(m+1) / 2 + 1$ points for each model instead of $2^m$. Compare with Figure \ref{fig:hull}. Left: iris dataset with the operating context $\theta_1$. Right: Pima Indian diabetes dataset with the operating context $\theta_2$.}
\label{fig:btc}
\end{figure}

\begin{figure}%[ht]
\centering
%\hspace{1cm}
\includegraphics[width=0.48\textwidth]{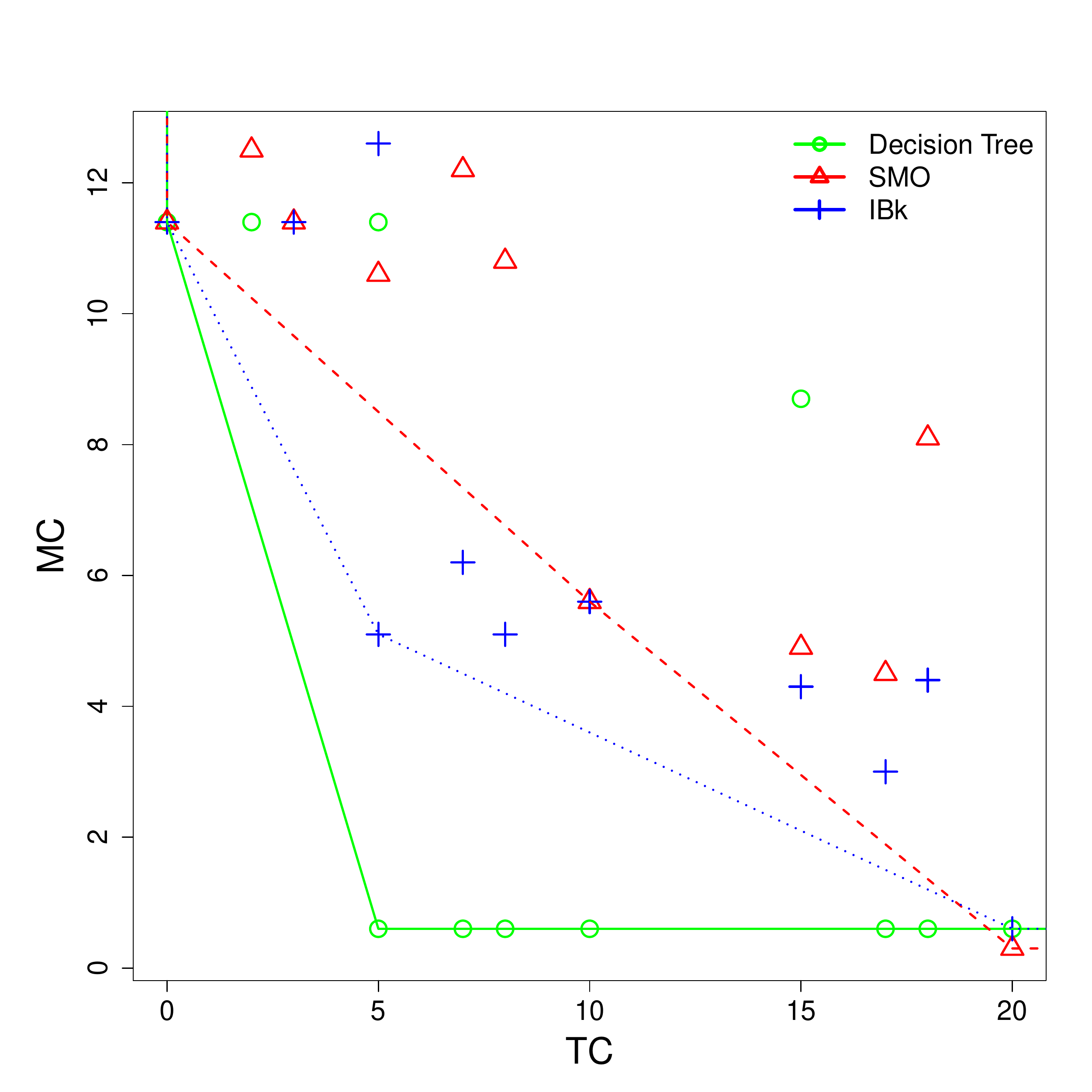} \hfill
\includegraphics[width=0.48\textwidth]{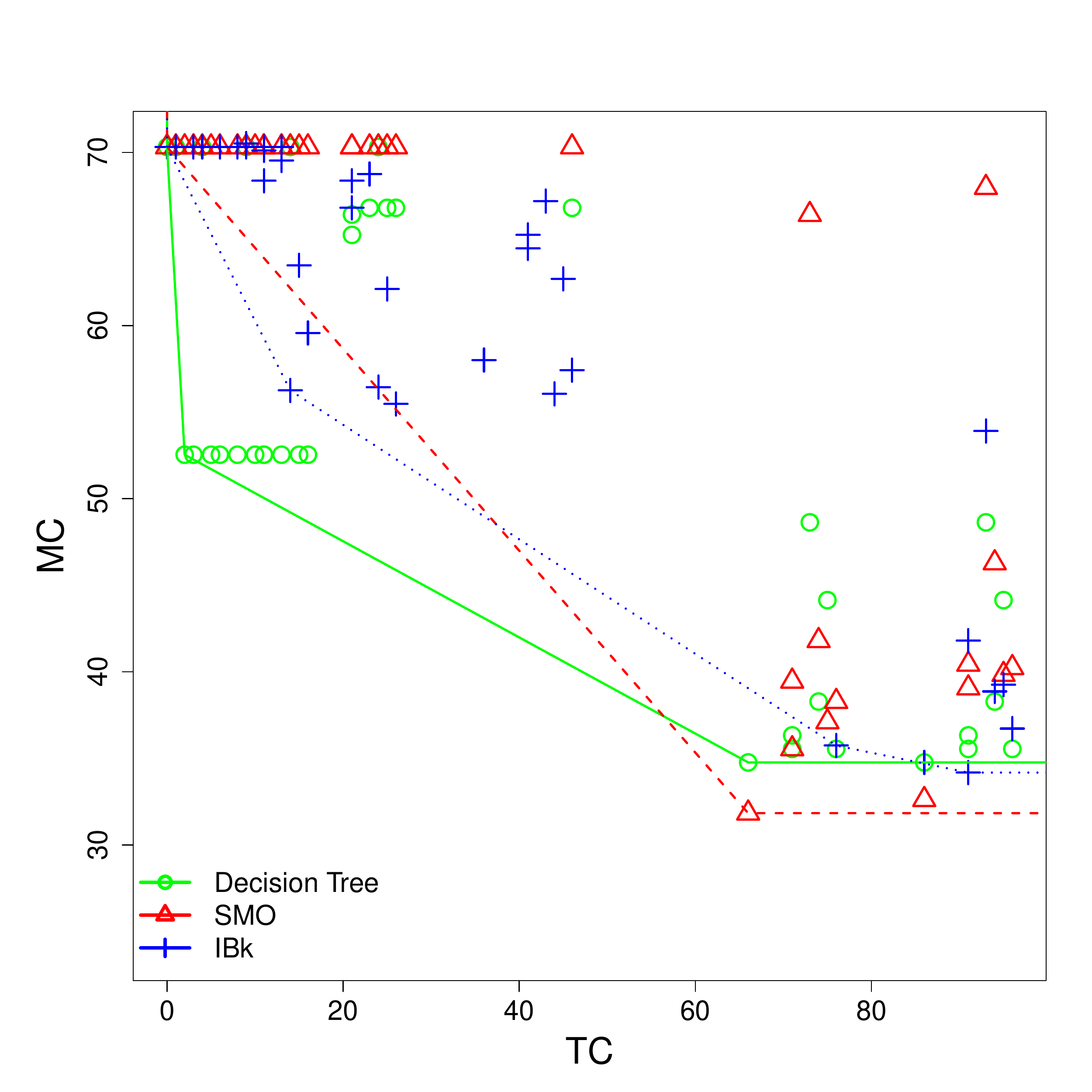} %\hspace{1cm}
\caption{We show the points and the hull of the combinations selected by the method BJC. There are $m(m+1) / 2 + 1$ points for each model instead of $2^m$. Compare with Figure \ref{fig:hull}. Left: iris dataset with the operating context $\theta_1$. Right: Pima Indian diabetes dataset with the operating context $\theta_2$.}
\label{fig:bjc}
\end{figure}

\begin{figure}%[ht]
\centering
%\hspace{1cm}
\includegraphics[width=0.48\textwidth]{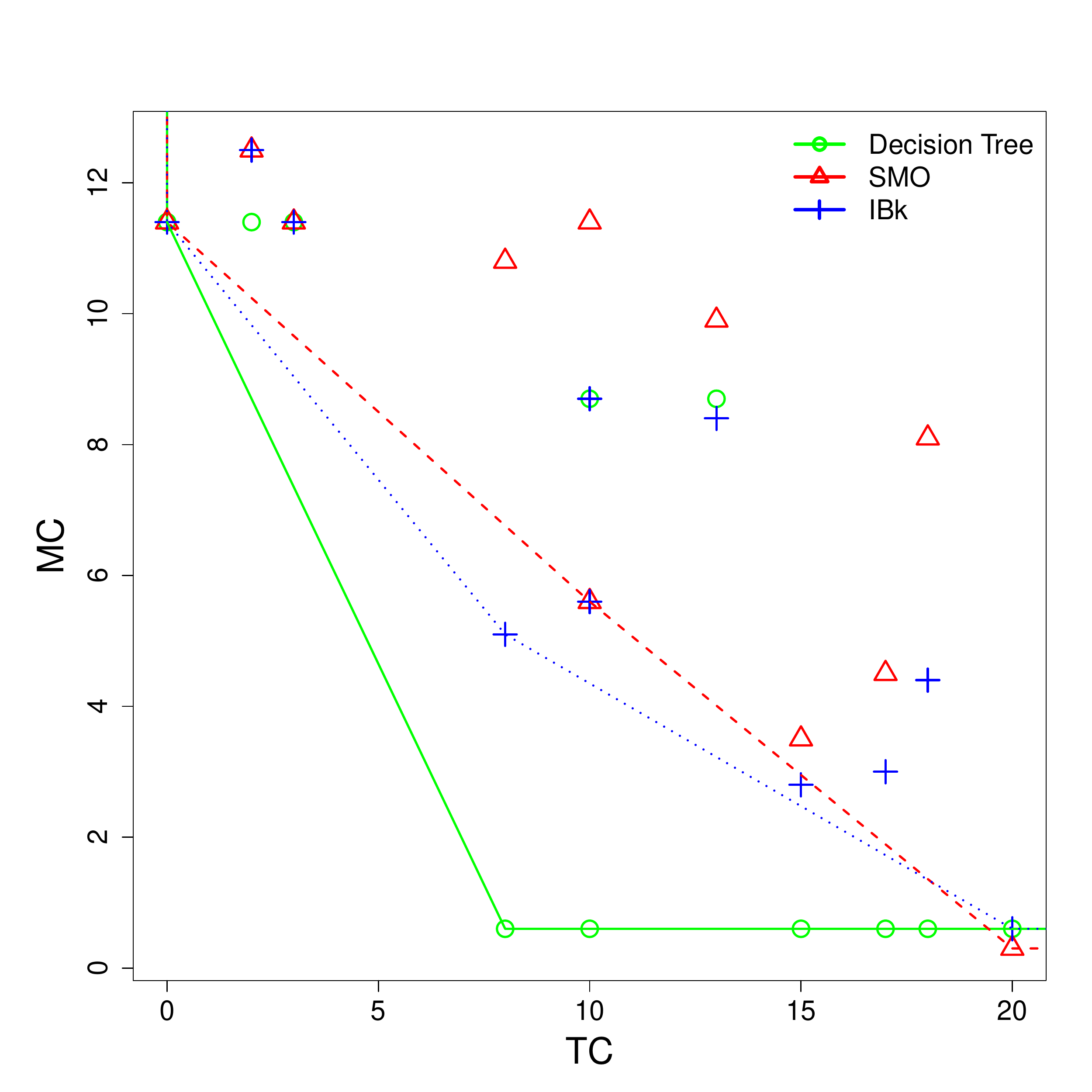} \hfill
\includegraphics[width=0.48\textwidth]{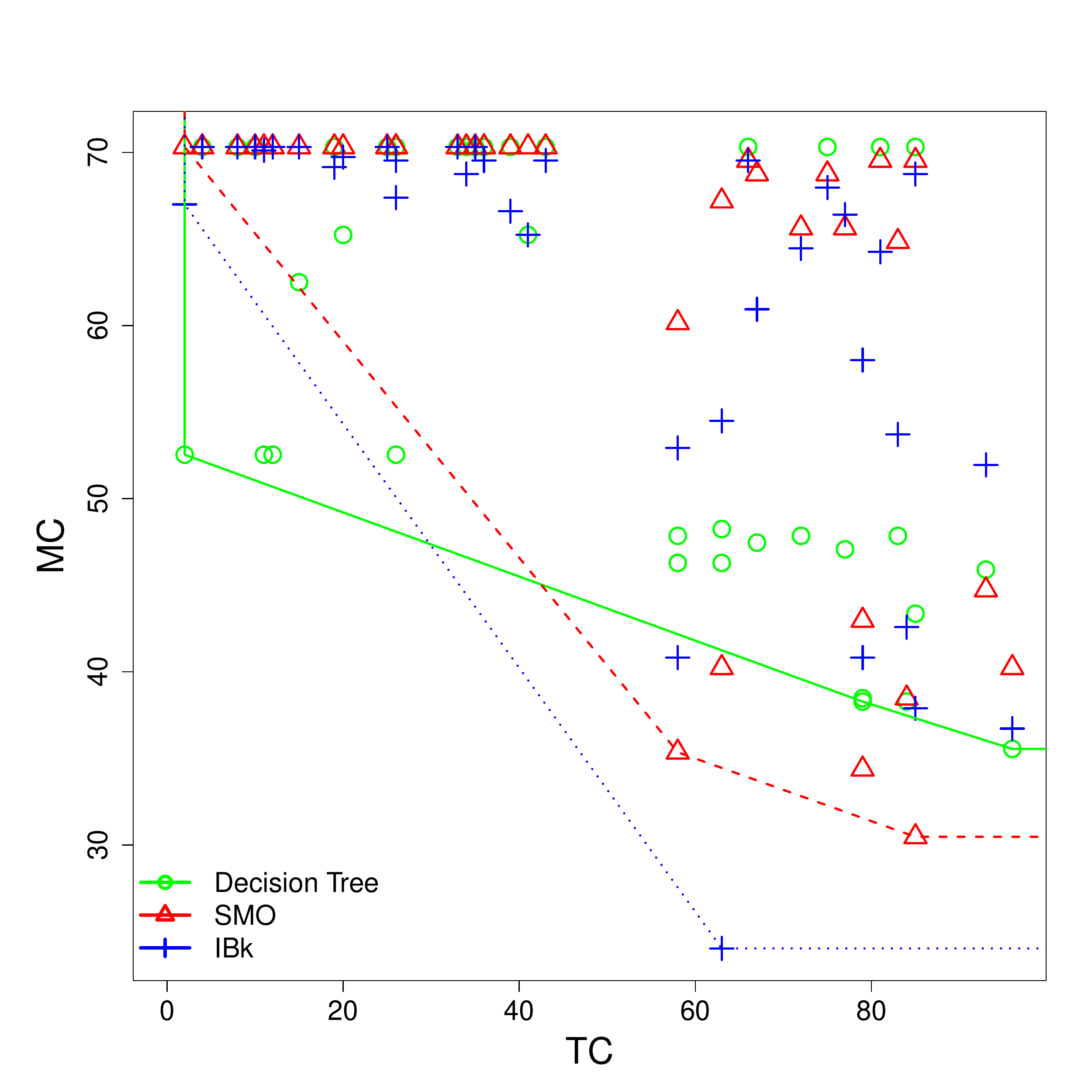} %\hspace{1cm}
\caption{We show the points and the hull of the combinations selected by the method RND. There are $m(m+1) / 2 + 1$ points for each model instead of $2^m$. Compare with Figure \ref{fig:hull}. Left: iris dataset with the operating context $\theta_1$. Right: Pima Indian diabetes dataset with the operating context $\theta_2$.}
\label{fig:rnd}
\end{figure}

%%%%%%%%%%%%%%%%%%%%%%%%%%%%%%%%%%%%%%%%%%%%%%%%%%%%%%%%%%%%%%%%%%
%%%%%%%%%%%%%%%%%%%%%%%%%%%%%%%%%%%%%%%%%%%%%%%%%%%%%%%%%%%%%%%%%%
\section{Experiments}\label{experiments}
%%%%%%%%%%%%%%%%%%%%%%%%%%%%%%%%%%%%%%%%%%%%%%%%%%%%%%%%%%%%%%%%%%
%%%%%%%%%%%%%%%%%%%%%%%%%%%%%%%%%%%%%%%%%%%%%%%%%%%%%%%%%%%%%%%%%%

Now we are going to explore whether the \JROC plots are effective, and also whether their quadratic approximation suffers from a degradation.
In order to do that, we consider six datasets of the UCI repository, with number of attributes between 4 and 11, as shown in Table \ref{tab:datasets}. We could not use larger datasets in this first experiment because the Full method is too slow as the number of elements to explore in the lattice grows exponentially.

\begin{table}
\begin{center}
\begin{tabular}{clccc}
\#& Dataset name          & $m$& $n$ & $c$ \\ \hline
1 & iris          & 4  & 150 & 3 \\
2 & diabetes      & 8  & 768 & 2 \\
3 & balance-scale & 5  & 625 & 3 \\
4 & breast-w      & 11 & 320 & 2 \\
5 & breast-cancer & 10 & 286 & 2 \\
6 & glass         & 9  & 214 & 5 \\\hline
\end{tabular}
\caption{Description of the datasets used in the experiments.}\label{tab:datasets}
\end{center}
\end{table}

We consider two different contexts: a uniform context $\theta_u$ and a random context where each value of the misclassification cost matrix and test cost vector are obtained by multiplying the original value of the uniform context by $k$, where $k= e^{\beta \times (k_0-0.5)}$, $k_0$ is obtained as a random number between 0 and 1 using a uniform distribution, and $\beta$ is a factor of how irregular we want the vector and matrix to be. We set $\beta=10$ for the following experiments. Once this function is applied, the test cost vector $T$ and the misclassification cost matrix $M$ are normalised such that $\sum T = 1$ and $\sum M = c^2$.

For each dataset of size $n$, we split it into a work dataset ($2n/3$ of the data) and the remaining data ($n/3$) for test.
%With the work dataset, we perform $10$ cross-validation. For each of the 10 folds, we train the models with 9/10 of the data and calculate all the points (i.e., $TC$ and $MC$) according to the full method, and the BMC, BTC, BJC and RND methods for 1/10 of the data. After the ten fold... % THIS CAN'T BE DONE AS WE WON'T BE ABLE TO AVERAGE FOR BMC, BTC, BJC and RND. This only works for the full method.
With the work dataset, we perform a split of the work dataset into two halves. 
We train four models (SMO, IBk, Adaboost with J48, Bagging with J48) with the first half of the data ($n/3$) and calculate all the points (i.e., $TC$ and $MC$) according to the full method, and the BMC, BTC, BJC and RND methods with the other half.
Next, we choose 5 values of $\alpha \in \{ 0.1, 0.3, 0.5, 0.7, 0.9 \}$ and determine the best configuration of model and feature subset for each of the five methods. We use the configurations given by the five methods (Full, BMC, BTC, BJC and RND) for the test set and calculate the $JC$. 

We repeat the experiment 4 times. This gives us, $4 \times 5= 20$ results for each of the 4 methods for each of the 6 datasets.

\subsection{Uniform context}

First we give the results for the uniform context. Table \ref{tab:results} shows the mean and standard deviation of the results for each dataset and method. We see that Full cannot be improved by the other methods, as it explores all the possibilities. In general, the RND method is worse than the backward methods, except for dataset 1 (the smaller one, iris, where the number of explored configurations is $4 \times 5 + 1 = 11$ in front of a total of 16, which is not a big difference). In fact, for the big datasets, where the difference in explored configuration grows exponentially, we see that the backward methods get close to the Full methods, which gives support to these approximation.

% latex table generated in R 3.0.0 by xtable 1.7-1 package
% Wed May 29 16:07:37 2013
\begin{table}
\centering
\begin{tabular}{rlllll}
  \hline
Dataset & Full & BMC & BTC & BJC & RND \\ \hline
  1 & 0.166 $\pm$ 0.0713 & 0.222 $\pm$ 0.105 & 0.207 $\pm$ 0.0909 & 0.222 $\pm$ 0.105 & 0.181 $\pm$ 0.0779 \\ 
  2 & 0.124 $\pm$ 0.0358 & 0.138 $\pm$ 0.0487 & 0.132 $\pm$ 0.0436 & 0.138 $\pm$ 0.0487 & 0.192 $\pm$ 0.0429 \\ 
  3 & 0.281 $\pm$ 0.156 & 0.286 $\pm$ 0.158 & 0.29 $\pm$ 0.162 & 0.286 $\pm$ 0.158 & 0.344 $\pm$ 0.122 \\ 
  4 & 0.289 $\pm$ 0.138 & 0.295 $\pm$ 0.142 & 0.293 $\pm$ 0.142 & 0.295 $\pm$ 0.142 & 0.358 $\pm$ 0.097 \\ 
  5 & 0.303 $\pm$ 0.123 & 0.328 $\pm$ 0.122 & 0.333 $\pm$ 0.13 & 0.328 $\pm$ 0.122 & 0.374 $\pm$ 0.0943 \\ 
  6 & 0.275 $\pm$ 0.129 & 0.28 $\pm$ 0.132 & 0.278 $\pm$ 0.131 & 0.28 $\pm$ 0.132 & 0.297 $\pm$ 0.126 \\ 
   \hline
\end{tabular}
\caption{JC results (mean and standard deviation) of the 20 results (5 values of alpha with 4 repetitions) for each of the 5 methods (columns: Full, BMC, BTC, BJC, RND) and each of the 6 datasets (rows: 1 to 6), with the uniform context.}\label{tab:results} 
\end{table}

Table \ref{tab:results2} shows the results aggregated for all datasets but showing each value of $\alpha$. This means ($8$ datasets with $10$ repetitions). In this case, we can see that the backward methods are consistently better than the RND method and are reasonable close to the Full method. The influence of $\alpha$ is not particularly clear, the approximation is similar for all of them.

% latex table generated in R 3.0.0 by xtable 1.7-1 package
% Wed May 29 16:07:37 2013
\begin{table}
\centering
\begin{tabular}{rlllll}
  \hline
$\alpha$ & Full & BMC & BTC & BJC & RND \\ \hline
  0.1 & 0.0776 $\pm$ 0.0175 & 0.082 $\pm$ 0.0219 & 0.082 $\pm$ 0.0219 & 0.082 $\pm$ 0.0219 & 0.187 $\pm$ 0.0726 \\ 
  0.3 & 0.204 $\pm$ 0.0366 & 0.233 $\pm$ 0.0476 & 0.22 $\pm$ 0.0482 & 0.233 $\pm$ 0.0476 & 0.259 $\pm$ 0.0696 \\ 
  0.5 & 0.294 $\pm$ 0.0914 & 0.329 $\pm$ 0.0919 & 0.319 $\pm$ 0.0867 & 0.329 $\pm$ 0.0919 & 0.335 $\pm$ 0.085 \\ 
  0.7 & 0.33 $\pm$ 0.121 & 0.345 $\pm$ 0.115 & 0.352 $\pm$ 0.122 & 0.345 $\pm$ 0.115 & 0.359 $\pm$ 0.122 \\ 
  0.9 & 0.292 $\pm$ 0.155 & 0.302 $\pm$ 0.15 & 0.305 $\pm$ 0.154 & 0.302 $\pm$ 0.15 & 0.315 $\pm$ 0.163 \\ 
   \hline
\end{tabular}
\caption{JC results (mean and standard deviation) of the 24 results
(6 datasets with 4 repetitions) 
for each of the 5 methods (columns: Full, BMC, BTC, BJC, RND) and each of the 5 possible values of $\alpha$ (rows 0.1 to 0.9), with the uniform context.}\label{tab:results2}
\end{table}

Finally, we want to see the whole picture and perform a statistical test.
% Este paragraf es igual que el article de probreg
In order to assess the significance of the experimental results we will use a custom procedure, following \cite{japkowicz2011evaluating} and \cite[ch.12]{petersbook}, which in turn is mostly based on \cite{demsar2006statistical}. Since we will not have any baseline method, we will use a Friedman test to tell whether the difference between several methods is significant and then we will apply the Nemenyi post-hoc test. We agree with \cite{garcia2008extension} that the Nemenyi test is a ``very conservative procedure and many of the obvious differences may not be detected", but we prefer to be conservative given our experimental setting and the use of a 0.95 confidence level. In some result tables we will show the means (even though in many cases they are not commensurate) and in some other tables we will show the average ranks (from which the Friedman and Nemenyi tests are calculated). We will also include the critical difference for the Nemenyi test, so we will be able to simply tell whether the difference between two algorithms is significant if the difference between their average ranks is greater than the critical difference.

Table \ref{tab:results3} shows the results of several data, where we are particularly interested in knowing which of the three backward methods is best. As we can see, the three methods behave almost equally. In fact, BMC and BJC are exactly equal, which is a consequence of the use of a uniform context.

% latex table generated in R 2.15.2 by xtable 1.7-1 package
% Wed May 29 20:01:45 2013
\begin{table}
\centering
\begin{tabular}{rrrrrr}
  \hline
 & Full & BMC & BTC & BJC & RND \\ 
  \hline
1 & 0.0953 & 0.1058 & 0.1058 & 0.1058 & 0.1388 \\ 
  2 & 0.2122 & 0.2767 & 0.2642 & 0.2767 & 0.2122 \\ 
  3 & 0.2250 & 0.3362 & 0.3013 & 0.3362 & 0.2562 \\ 
  4 & 0.2047 & 0.2715 & 0.2528 & 0.2715 & 0.2047 \\ 
  5 & 0.0905 & 0.1217 & 0.1092 & 0.1217 & 0.0915 \\ 
  6 & 0.0674 & 0.0674 & 0.0674 & 0.0674 & 0.2220 \\ 
  7 & 0.1532 & 0.2001 & 0.1581 & 0.2001 & 0.2023 \\ 
  8 & 0.1504 & 0.1649 & 0.1660 & 0.1649 & 0.2146 \\ 
  9 & 0.1340 & 0.1340 & 0.1340 & 0.1340 & 0.1834 \\ 
  10 & 0.1169 & 0.1238 & 0.1338 & 0.1238 & 0.1397 \\ 
  11 & 0.0542 & 0.0542 & 0.0542 & 0.0542 & 0.2028 \\ 
  12 & 0.1797 & 0.1797 & 0.1797 & 0.1797 & 0.2380 \\ 
  13 & 0.3125 & 0.3264 & 0.3264 & 0.3264 & 0.3595 \\ 
  14 & 0.4002 & 0.4036 & 0.4068 & 0.4036 & 0.4429 \\ 
  15 & 0.4571 & 0.4649 & 0.4807 & 0.4649 & 0.4774 \\ 
  16 & 0.0738 & 0.0738 & 0.0738 & 0.0738 & 0.2090 \\ 
  17 & 0.2074 & 0.2074 & 0.2074 & 0.2074 & 0.3211 \\ 
  18 & 0.3262 & 0.3457 & 0.3311 & 0.3457 & 0.3779 \\ 
  19 & 0.4186 & 0.4277 & 0.4295 & 0.4277 & 0.4416 \\ 
  20 & 0.4168 & 0.4213 & 0.4217 & 0.4213 & 0.4393 \\ 
  21 & 0.0975 & 0.1117 & 0.1117 & 0.1117 & 0.2225 \\ 
  22 & 0.2419 & 0.3001 & 0.2790 & 0.3001 & 0.3511 \\ 
  23 & 0.3736 & 0.4035 & 0.3958 & 0.4035 & 0.4188 \\ 
  24 & 0.3904 & 0.3988 & 0.4521 & 0.3988 & 0.4321 \\ 
  25 & 0.4137 & 0.4254 & 0.4273 & 0.4254 & 0.4454 \\ 
  26 & 0.0775 & 0.0793 & 0.0793 & 0.0793 & 0.1282 \\ 
  27 & 0.2309 & 0.2309 & 0.2309 & 0.2309 & 0.2309 \\ 
  28 & 0.3792 & 0.3980 & 0.3916 & 0.3980 & 0.3829 \\ 
  29 & 0.4346 & 0.4346 & 0.4346 & 0.4346 & 0.4475 \\ 
  30 & 0.2550 & 0.2550 & 0.2550 & 0.2550 & 0.2974 \\  \hline
  Avg & 0.2397 & 0.2581 & 0.2554 & 0.2581 & 0.2911 \\  \hline
  AR & 1.5000 & 3.0500 & 3.0667 & 3.0500 & 4.3333 \\ 
   \hline
\end{tabular}
\caption{This figure shows the JC means for the 4 repetitions for each of the 5 methods (Full, BMC, BTC, BJC, RND). The 30 rows are given by 6 datasets and 5 possible values of $\alpha$ with the uniform context. The `Avg' row shows the averages of the first 30 rows. Finally, the `AR' row shows the average rank for each method. With these ranks the Friedman test is applied and gives a Friedman statistic of 62.51 which is greather than the critical value of 10.97. Consequently, the null hypothesis is rejected (significance level: 0.05) and we conclude that the methods do not perform equally. In order to see which methods are significantly different from the rest, we look at the critical difference for the Nemenyi post-hoc test, which is 0.2965. This means that the Full method is statistically better than the rest, and that the RND method is statistically worse than the rest, but there is no statistically significant difference between the three methods BMC, BTC and BJC.}\label{tab:results3}
\end{table}

\subsection{Variable context}

In order to see what happens in a more realistic situation, let us see the results for the non-uniform context. Table \ref{tab:results-UR} shows the mean and standard deviation of the results for each dataset and method. Here we see that not all backward methods are equivalently, but interestingly we see that BMC is now consistently better than RND for all datasets.

% latex table generated in R 3.0.0 by xtable 1.7-1 package
% Wed May 29 16:07:37 2013
\begin{table}
\centering
\begin{tabular}{rlllll}
  \hline
Dataset & Full & BMC & BTC & BJC & RND \\ \hline
  1 & 0.0079 $\pm$ 0.011 & 0.011 $\pm$ 0.013 & 0.015 $\pm$ 0.025 & 0.018 $\pm$ 0.028 & 0.017 $\pm$ 0.04 \\ 
  2 & 0.016 $\pm$ 0.016 & 0.023 $\pm$ 0.024 & 0.016 $\pm$ 0.017 & 0.025 $\pm$ 0.033 & 0.032 $\pm$ 0.033 \\ 
  3 & 0.037 $\pm$ 0.035 & 0.039 $\pm$ 0.035 & 0.037 $\pm$ 0.035 & 0.039 $\pm$ 0.036 & 0.045 $\pm$ 0.044 \\ 
  4 & 0.045 $\pm$ 0.049 & 0.052 $\pm$ 0.057 & 0.05 $\pm$ 0.057 & 0.052 $\pm$ 0.057 & 0.058 $\pm$ 0.057 \\ 
  5 & 0.0026 $\pm$ 0.0036 & 0.004 $\pm$ 0.0056 & 0.0033 $\pm$ 0.0037 & 0.0034 $\pm$ 0.004 & 0.0048 $\pm$ 0.0046 \\ 
  6 & 0.012 $\pm$ 0.017 & 0.02 $\pm$ 0.028 & 0.024 $\pm$ 0.041 & 0.024 $\pm$ 0.041 & 0.023 $\pm$ 0.03 \\ 
   \hline
\end{tabular}
\caption{JC results (mean and standard deviation) of the 20 results (5 values of alpha with 4 repetitions) for each of the 5 methods (columns: Full, BMC, BTC, BJC, RND) and each of the 6 datasets (rows: 1 to 6), with the variable context.}\label{tab:results-UR} 
\end{table}

Again, Table \ref{tab:results2-UR} shows the results aggregated for all datasets but showing each value of $\alpha$. This means ($8$ datasets with $10$ repetitions). Now we see that the results are not especially different according to $\alpha$.

% latex table generated in R 3.0.0 by xtable 1.7-1 package
% Wed May 29 16:07:37 2013
\begin{table}
\centering
\begin{tabular}{rlllll}
  \hline
$\alpha$ & Full & BMC & BTC & BJC & RND \\ \hline
  0.1 & 0.0055 $\pm$ 0.0064 & 0.0081 $\pm$ 0.01 & 0.006 $\pm$ 0.0067 & 0.0061 $\pm$ 0.0067 & 0.012 $\pm$ 0.021 \\ 
  0.3 & 0.014 $\pm$ 0.021 & 0.018 $\pm$ 0.028 & 0.015 $\pm$ 0.023 & 0.022 $\pm$ 0.034 & 0.024 $\pm$ 0.032 \\ 
  0.5 & 0.025 $\pm$ 0.034 & 0.037 $\pm$ 0.047 & 0.04 $\pm$ 0.054 & 0.042 $\pm$ 0.055 & 0.046 $\pm$ 0.06 \\ 
  0.7 & 0.022 $\pm$ 0.027 & 0.025 $\pm$ 0.028 & 0.024 $\pm$ 0.028 & 0.025 $\pm$ 0.028 & 0.031 $\pm$ 0.032 \\ 
  0.9 & 0.033 $\pm$ 0.044 & 0.036 $\pm$ 0.044 & 0.035 $\pm$ 0.044 & 0.039 $\pm$ 0.045 & 0.038 $\pm$ 0.045 \\ 
   \hline
\end{tabular}
\caption{JC results (mean and standard deviation) of the 24 results
(6 datasets with 4 repetitions) 
for each of the 5 methods (columns: Full, BMC, BTC, BJC, RND) and each of the 5 possible values of $\alpha$ (rows 0.1 to 0.9), with the variable context.}\label{tab:results2-UR}
\end{table}

Finally, if we look at the whole picture and using a statistical test, we see in Table \ref{tab:results3-UR} that the backward methods are better than the RND method, but now we find difference between them. In fact, BTC is significantly better than BMC and BJC.

% latex table generated in R 2.15.2 by xtable 1.7-1 package
% Wed May 29 20:01:45 2013
\begin{table}
\centering
\begin{tabular}{rrrrrr}
  \hline
 & Full & BMC & BTC & BJC & RND \\ 
  \hline
1 & 0.0023 & 0.0025 & 0.0025 & 0.0025 & 0.0025 \\ 
  2 & 0.0071 & 0.0071 & 0.0073 & 0.0073 & 0.0100 \\ 
  3 & 0.0160 & 0.0258 & 0.0457 & 0.0457 & 0.0507 \\ 
  4 & 0.0085 & 0.0114 & 0.0117 & 0.0117 & 0.0087 \\ 
  5 & 0.0054 & 0.0057 & 0.0057 & 0.0210 & 0.0124 \\ 
  6 & 0.0078 & 0.0205 & 0.0078 & 0.0082 & 0.0338 \\ 
  7 & 0.0138 & 0.0306 & 0.0138 & 0.0538 & 0.0482 \\ 
  8 & 0.0141 & 0.0191 & 0.0145 & 0.0157 & 0.0267 \\ 
  9 & 0.0171 & 0.0185 & 0.0171 & 0.0180 & 0.0258 \\ 
  10 & 0.0248 & 0.0279 & 0.0254 & 0.0277 & 0.0267 \\ 
  11 & 0.0116 & 0.0129 & 0.0129 & 0.0129 & 0.0171 \\ 
  12 & 0.0358 & 0.0361 & 0.0358 & 0.0358 & 0.0398 \\ 
  13 & 0.0594 & 0.0611 & 0.0594 & 0.0606 & 0.0772 \\ 
  14 & 0.0563 & 0.0585 & 0.0580 & 0.0585 & 0.0665 \\ 
  15 & 0.0206 & 0.0240 & 0.0207 & 0.0254 & 0.0241 \\ 
  16 & 0.0098 & 0.0098 & 0.0098 & 0.0098 & 0.0156 \\ 
  17 & 0.0243 & 0.0299 & 0.0290 & 0.0274 & 0.0343 \\ 
  18 & 0.0481 & 0.0728 & 0.0645 & 0.0728 & 0.0804 \\ 
  19 & 0.0253 & 0.0253 & 0.0253 & 0.0253 & 0.0310 \\ 
  20 & 0.1191 & 0.1220 & 0.1219 & 0.1223 & 0.1264 \\ 
  21 & 0.0003 & 0.0006 & 0.0003 & 0.0005 & 0.0010 \\ 
  22 & 0.0012 & 0.0026 & 0.0025 & 0.0026 & 0.0065 \\ 
  23 & 0.0013 & 0.0028 & 0.0037 & 0.0023 & 0.0039 \\ 
  24 & 0.0037 & 0.0054 & 0.0039 & 0.0052 & 0.0047 \\ 
  25 & 0.0062 & 0.0084 & 0.0063 & 0.0065 & 0.0078 \\ 
  26 & 0.0014 & 0.0026 & 0.0026 & 0.0026 & 0.0014 \\ 
  27 & 0.0025 & 0.0035 & 0.0035 & 0.0035 & 0.0025 \\ 
  28 & 0.0123 & 0.0387 & 0.0548 & 0.0548 & 0.0352 \\ 
  29 & 0.0225 & 0.0278 & 0.0288 & 0.0294 & 0.0474 \\ 
  30 & 0.0216 & 0.0270 & 0.0283 & 0.0283 & 0.0308 \\ \hline
  Avg & 0.0200 & 0.0247 & 0.0241 & 0.0266 & 0.0300 \\  \hline
  AR & 1.2333 & 3.4000 & 2.6500 & 3.5167 & 4.2000 \\ 
   \hline
\end{tabular}
\caption{This figure shows the JC means for the 4 repetitions for each of the 5 methods (Full, BMC, BTC, BJC, RND). The 30 rows are given by 6 datasets and 5 possible values of $\alpha$ with the variable context. The `Avg' row shows the averages of the first 30 rows. Finally, the `AR' row shows the average rank for each method. With these ranks the Friedman test is applied and gives a Friedman statistic of 67.27 which is greater than the critical value of 10.97. Consequently, the null hypothesis is rejected (significance level: 0.05) and we conclude that the methods do not perform equally. In order to see which methods are significantly different from the rest, we look at the critical difference for the Nemenyi post-hoc test, which is 0.2965. This means that the Full method is statistically better than the rest, and that the RND method is statistically worse than the rest. In this case, we see that the BTC method is significantly better than BMC and BJC.}\label{tab:results3-UR}
\end{table}

Although some more definitive conclusions of which method is best in general would require more datasets and repetitions (although the results are significant enough here), these experiments show the potential of the backward methods.

%%%%%%%%%%%%%%%%%%%%%%%%%%%%%%%%%%%%%%%%%%%%%%%%%%%%%%%%%%%%%%%%%%
%%%%%%%%%%%%%%%%%%%%%%%%%%%%%%%%%%%%%%%%%%%%%%%%%%%%%%%%%%%%%%%%%%
\section{Conclusion}\label{conclusion}
%%%%%%%%%%%%%%%%%%%%%%%%%%%%%%%%%%%%%%%%%%%%%%%%%%%%%%%%%%%%%%%%%%
%%%%%%%%%%%%%%%%%%%%%%%%%%%%%%%%%%%%%%%%%%%%%%%%%%%%%%%%%%%%%%%%%%

In the introduction we argued that we were looking for a flexible approach that could be used in a variety of circumstances. In fact, we were motivated by  the following considerations:

\begin{enumerate}
\item The method must work for any kind of predictive model, either human-made or trained from data using any off-the-shelf predictive modelling technique. \label{enu:any}
\item Each example may have a different subset of missing values. \label{enu:instances}
\item Retraining the model for each example (using a subset of the examples with similar feature subsets) is not an option (because of \ref{enu:instances} above or other reasons).
\item Both misclassification cost (MC) and test cost (TC) must be considered.
\end{enumerate}

\noindent We have presented some graphical tools and an optimisation method that meets these requirements. In fact, this has to be compared to the usual approach which is specific on decision trees, with several approaches according to \cite{zhang2005missing,lomax2013survey}: % \todo{READ lomax}: 
(a) KV, a tree is rebuilt when missing values are found, (b) Null strategy: replace by an extra label (model is not rebuilt), (c) Internal node: creates nodes for examples with missing values (model is not rebuilt), and (d) C4.5 strategy: probabilistic approach (model is not rebuilt). Option (a) is infeasible if the situation \ref{enu:instances} holds. 

All the above options are specific to decision trees, so they are not able to take advantage of many other off-the-shelf techniques of our preferred data mining suite or machine learning library. This is an important limitation as many of the most powerful machine learning techniques used today, such as ensemble methods (using or not decision trees as base classifiers), support vector machines, etc., are much more difficult to adapt for minimising test costs.

The take-away message of this paper is that we can use any machine learning technique, train a model on a dataset with the available attributes and possibly containing missing values, and {\em reframe} it for a different deployment context where we have fewer available attributes, a different distribution of missing values, a different misclassification matrix and test cost vector. While the Full approach is intractable in general, we have introduced some approximations that are just quadratic, which are feasible for hundreds of attributes, which is already a high number of attributes if we are considering test costs. Also, during all the process we can explore the performance of several models using \JROC curves. In fact, these curves are not specific for the methods we have introduced here; they could be used for the traditional methods used for decision trees or for the analysis of any cost context considering both $MC$ and $TC$.

This work opens many new avenues of future work. For instance, in section \ref{reframing} we discuss that an alternative to the use of missing values is the use of ranges (see bullet \ref{item-range}). The approach presented here could also be compared or explored in combination to the mimetic technique to get models that use fewer attributes \cite{ensemble2002,mimetic2003,Blanco-VegaHR04,icmlaBlanco-VegaHR05,blancoestimating}. Another interesting idea would be the problem of quantification with test costs, which could be applied to both classification and regression \cite{bella2010quantification,bella2013aggregative}. 
We have also been suggested\footnote{Peter Flach, personal communication} to use decision stump ensembles, where the elements in 
the ensemble could be pruned a posteriori when the test cost is known.

More comparison with the area of feature selection could lead to a better understanding of the possibilities of reframing and better methods. For instance, the use of the attribute correlation can be used to an approximate notion of dominance (e.g., if two attributes have high correlation, the cost is expected to be related to the lowest test cost for any of them). As for the relation to other problems, we could also consider that the output domain may be null, as in abstaining classifiers \cite{ferri2004cautious,Pietraszek2005} and the notion of delegation \cite{FerriFH04} could be applied to this case. In fact, a missing value on purpose can be seen as the parallel of a reject option or abstention for the output value.

The notion of \JROC curve could be further explored and extended. For instance, we could figure out other ways of drawing these curves, by using attribute correlation or some other order on the attributes. The issue of representing operating conditions when the the matrix and vector are not fixed could lead to more dimensions, or the inclusion of the cost matrix. At least in the case of binary classifiers we could have 3D surfaces, using, e.g., $TPR$, $FPR$ and $TC$. 
As for any curve representing cost for each operating condition, we wonder whether the area over the \JROC curve means something, as in ROC analysis \cite{Bradley1997,ICML11CoherentAUC}. Also we could ask the question of whether we can draw cost plots as in \cite{drummond-and-Holte2006,ICML11Brier,ROCandCost}.

Finally, there are more more ambitious ideas. We could investigate which attributes to use for each example. We could use reliability measures (especially in probabilistic classifiers) to make better decisions on whether to remove an attribute or not. We could analyse what to do when new attributes appear, using, e.g., the correlation to other attributes to derive the old attributes, or thinking about more general ways of representing the feature space. Finally, we think that there is no reason why most of the ideas introduced here could not work equally well for regression, combining the test cost with any regression loss.

%\setstretch{.9}  % REQUIRES \usepackage{setspace}

%%%%%%%%%%%%%%%%%%%%%%%%%%%%%%%%%%%%%%%%%%%%%%%%%%%%%%%%%%%%%%%%%%
%%%%%%%%%%%%%%%%%%%%%%%%%%%%%%%%%%%%%%%%%%%%%%%%%%%%%%%%%%%%%%%%%%
\section*{Acknowledgements}
%%%%%%%%%%%%%%%%%%%%%%%%%%%%%%%%%%%%%%%%%%%%%%%%%%%%%%%%%%%%%%%%%%
%%%%%%%%%%%%%%%%%%%%%%%%%%%%%%%%%%%%%%%%%%%%%%%%%%%%%%%%%%%%%%%%%%

{\small This work was supported by the MEC/MINECO projects CONSOLIDER-INGENIO CSD2007-00022 and TIN 2010-21062-C02-02, GVA project Prometeo/2008/051, the COST - European Cooperation in the field of Scientific and Technical Research IC0801 AT, and the REFRAME project granted by the European Coordinated Research on Long-term Challenges in Information and Communication Sciences \& Technologies ERA-Net (CHIST-ERA), and funded by the respective national research councils and ministries.}

\bibliography{biblio}

\begin{thebibliography}{10}

\bibitem{UCIrep2013}
K.~Bache and M.~Lichman.
\newblock {UCI} machine learning repository, 2013.

\bibitem{bella2010quantification}
A.~Bella, C.~Ferri, J.~Hern{\'a}ndez-Orallo, and M.~J. Ram{\'\i}rez-Quintana.
\newblock Quantification via probability estimators.
\newblock In {\em 2010 IEEE International Conference on Data Mining}, pages
  737--742. IEEE, 2010.

\bibitem{bella2013aggregative}
A.~Bella, C.~Ferri, J.~Hern{\'a}ndez-Orallo, and M.~J. Ram{\'\i}rez-Quintana.
\newblock Aggregative quantification for regression.
\newblock {\em Data Mining and Knowledge Discovery}, pages 1--44, 2013.

\bibitem{bella2011using}
A.~Bella, C.~Ferri, J.~Hern{\'a}ndez-Orallo, and M.J. Ram{\'\i}rez-Quintana.
\newblock Using negotiable features for prescription problems.
\newblock {\em Computing}, 91(2):135--168, 2011.

\bibitem{blancoestimating}
R.~Blanco-Vega, C.~Ferri-Ram{\'\i}rez, J.~Hern{\'a}ndez-Orallo, and M.J.
  Ram{\'\i}rez-Quintana.
\newblock Estimating the class probability threshold without training data.
\newblock {\em Cèsar Ferri, Nicolas Lachiche, Sofus A. Macskassy (eds.) ROC
  Analysis in Machine Learning, ICML2006 workshop, Pittsburgh, USA}, page~9,
  2006.

\bibitem{Blanco-VegaHR04}
R.~Blanco-Vega, J.~Hern{\'a}ndez-Orallo, and M.~J. Ram\'{\i}rez-Quintana.
\newblock Analysing the trade-off between comprehensibility and accuracy in
  mimetic models.
\newblock In {\em Discovery Science, 7th International Conference, DS 2004,
  Padova, Italy, October 2-5, 2004, Proceedings}, volume 3245 of {\em Lecture
  Notes in Computer Science}, pages 338--346. Springer, 2004.

\bibitem{icmlaBlanco-VegaHR05}
R.~Blanco-Vega, J.~Hern{\'a}ndez-Orallo, and M.~J. Ram\'{\i}rez-Quintana.
\newblock Knowledge acquisition through machine learning: minimising expert's
  effort.
\newblock In M.~Arif Wani, Mariofanna~G. Milanova, Lukasz~A. Kurgan, Marek
  Reformat, and Khalid Hafeez, editors, {\em Fourth International Conference on
  Machine Learning and Applications, ICMLA 2005, Los Angeles, California, USA,
  15-17 December 2005}. IEEE Computer Society, 2005.

\bibitem{Bradley1997}
A.~P. Bradley.
\newblock The use of the area under the {ROC} curve in the evaluation of
  machine learning algorithms.
\newblock {\em Pattern Recognition}, 30(7):1145 -- 1159, 1997.

\bibitem{demsar2006statistical}
J.~Dem{\v{s}}ar.
\newblock Statistical comparisons of classifiers over multiple data sets.
\newblock {\em The Journal of Machine Learning Research}, 7:1--30, 2006.

\bibitem{drummond-and-Holte2006}
C.~Drummond and R.C. Holte.
\newblock {Cost Curves: An Improved Method for Visualizing Classifier
  Performance}.
\newblock {\em Machine Learning}, 65:95--130, 2006.

\bibitem{elkan2001foundations}
C.~Elkan.
\newblock The foundations of cost-sensitive learning.
\newblock In {\em International Joint Conference on Artificial Intelligence},
  pages 973--978. Citeseer, 2001.

\bibitem{mimetic2003}
V.~Estruch, C.~Ferri, J.~Hernandez-Orallo, and M.~Rami­rez-Quintana.
\newblock Simple mimetic classifiers.
\newblock In Petra Perner and Azriel Rosenfeld, editors, {\em Machine Learning
  and Data Mining in Pattern Recognition}, volume 2734 of {\em Lecture Notes in
  Computer Science}, pages 197--227. Springer Berlin / Heidelberg, 2003.
\newblock 10.1007/3-540-45065-3\_14.

\bibitem{Fawcett06}
T.~Fawcett.
\newblock An introduction to {ROC} analysis.
\newblock {\em Pattern Recognition Letters}, 27(8):861--874, 2006.

\bibitem{FerriFH04}
C.~Ferri, P.~A. Flach, and J.~Hern{\'a}ndez-Orallo.
\newblock Delegating classifiers.
\newblock In Carla~E. Brodley, editor, {\em Machine Learning, Proceedings of
  the Twenty-first International Conference (ICML 2004), Banff, Alberta,
  Canada, July 4-8, 2004}, volume~69 of {\em ACM International Conference
  Proceeding Series}. ACM, 2004.

\bibitem{ferri2004cautious}
C.~Ferri and J.~Hern{\'a}ndez-Orallo.
\newblock Cautious classifiers.
\newblock {\em Proceedings of the 1st International Workshop on ROC Analysis in
  Artificial Intelligence (ROCAI-2004)}, pages 27--36, 2004.

\bibitem{PRL09}
C.~Ferri, J.~Hern\'{a}ndez-Orallo, and R.~Modroiu.
\newblock An experimental comparison of performance measures for
  classification.
\newblock {\em Pattern Recognition Let.}, 30(1):27--38, 2009.

\bibitem{ensemble2002}
C.~Ferri, J.~Hern{\'a}ndez-Orallo, and M.~Ram\'{\i}rez-Quintana.
\newblock From ensemble methods to comprehensible models.
\newblock In Steffen Lange, Ken Satoh, and Carl Smith, editors, {\em Discovery
  Science}, volume 2534 of {\em Lecture Notes in Computer Science}, pages
  223--234. Springer Berlin / Heidelberg, 2002.
\newblock 10.1007/3-540-36182-0\_16.

\bibitem{petersbook}
P.~Flach.
\newblock {\em Machine Learning: The Art and Science of Algorithms that Make
  Sense of Data}.
\newblock Cambridge University Press, 2012.

\bibitem{flach2003decision}
P.~Flach, H.~Blockeel, C.~Ferri, J.~Hern{\'a}ndez-Orallo, and J.~Struyf.
\newblock Decision support for data mining.
\newblock {\em Data Mining and Decision Support}, pages 81--90, 2003.

\bibitem{ICML11CoherentAUC}
P.~Flach, J.~Hern{\'a}ndez-Orallo, and C.~Ferri.
\newblock A coherent interpretation of {AUC} as a measure of aggregated
  classification performance.
\newblock In {\em Proceedings of the 28th International Conference on Machine
  Learning, ICML2011}, 2011.

\bibitem{garcia2008extension}
S.~Garc\'{\i}a and F.~Herrera.
\newblock An extension on statistical comparisons of classifiers over multiple
  data sets for all pairwise comparisons.
\newblock {\em The Journal of Machine Learning Research}, 9(2677-2694):66,
  2008.

\bibitem{weka}
M.~Hall, E.~Frank, G.~Holmes, B.~Pfahringer, P.~Reutemann, and I.~H. Witten.
\newblock The weka data mining software: an update.
\newblock {\em ACM SIGKDD Explorations Newsletter}, 11(1):10--18, 2009.

\bibitem{RROC2012}
J.~Hern{\'a}ndez-Orallo.
\newblock A graphical analysis of cost-sensitive regression problems.
\newblock {\em http://arxiv.org/abs/1211.2359}, 2012.

\bibitem{rocai2004}
J.~Hern{\'a}ndez-Orallo, C.~Ferri, N.~Lachiche, and P.~Flach.
\newblock The 1st workshop on {ROC} analysis in artificial intelligence
  {(ROCAI-2004)}.
\newblock {\em ACM SIGKDD Explorations Newsletter}, 6(2):159--161, 2004.

\bibitem{ICML11Brier}
J.~Hern\'{a}ndez-Orallo, P.~Flach, and C.~Ferri.
\newblock Brier curves: a new cost-based visualisation of classifier
  performance.
\newblock In {\em Proceedings of the 28th International Conference on Machine
  Learning, ICML2011}, 2011.

\bibitem{JMLR12}
J.~Hern{\'a}ndez-Orallo, P.~Flach, and C.~Ferri.
\newblock A unified view of performance metrics: Translating threshold choice
  into expected classification loss.
\newblock {\em Journal of Machine Learning Research (JMLR)}, 13:2813--2869,
  2012.

\bibitem{ROCandCost}
J.~Hern{\'a}ndez-Orallo, P.~Flach, and C.~Ferri.
\newblock {ROC} curves in cost space.
\newblock {\em Machine Learning}, 2013.

\bibitem{japkowicz2011evaluating}
N.~Japkowicz and M.~Shah.
\newblock {\em Evaluating Learning Algorithms: A Classification Perspective}.
\newblock Cambridge Univ Pr, 2011.

\bibitem{li2001feature}
H.~X. Li and L.~D. Xu.
\newblock Feature space theory - a mathematical foundation for data mining.
\newblock {\em Knowledge-based systems}, 14(5):253--257, 2001.

\bibitem{ling2004decision}
C.X. Ling, Q.~Yang, J.~Wang, and S.~Zhang.
\newblock Decision trees with minimal costs.
\newblock In {\em Proceedings of the twenty-first international conference on
  Machine learning}, page~69. ACM, 2004.

\bibitem{lomax2013survey}
S.~Lomax and S.l Vadera.
\newblock A survey of cost-sensitive decision tree induction algorithms.
\newblock {\em ACM Computing Surveys (CSUR)}, 45(2):16, 2013.

\bibitem{Mamitsuka2006}
H.~Mamitsuka.
\newblock Selecting features in microarray classification using {ROC} curves.
\newblock {\em Pattern Recognition}, 39(12):2393 -- 2404, 2006.

\bibitem{Pietraszek2005}
T.~Pietraszek.
\newblock Optimizing abstaining classifiers using {ROC} analysis.
\newblock In {\em Proceedings of the 22nd international conference on Machine
  learning}, ICML '05, pages 665--672, New York, NY, USA, 2005. ACM.

\bibitem{SDM00}
J.~A. Swets, R.~M. Dawes, and J.~Monahan.
\newblock Better decisions through science.
\newblock {\em Scientific American}, 283(4):82--87, October 2000.

\bibitem{turney2000types}
P.~Turney.
\newblock Types of cost in inductive concept learning.
\newblock {\em Canada National Research Council Publications Archive}, 2000.

\bibitem{zhang2005missing}
S.~Zhang, Z.~Qin, C.~X. Ling, and S.~Sheng.
\newblock Missing is useful: missing values in cost-sensitive decision trees.
\newblock {\em IEEE transactions on knowledge and data engineering},
  17(12):1689--1693, 2005.

\bibitem{zhu2011missing}
X.~Zhu, S.~Zhang, Z.~Jin, Z.~Zhang, and Z.~Xu.
\newblock Missing value estimation for mixed-attribute data sets.
\newblock {\em Knowledge and Data Engineering, IEEE Transactions on},
  23(1):110--121, 2011.

\end{thebibliography}

\end{document}